\documentclass[preprint,12pt,authoryear]{elsarticle}

\usepackage{amssymb}
\usepackage{amsmath}
\usepackage{booktabs}
\usepackage{bbm}
\usepackage{hyperref}
\usepackage{amsmath,amsthm,amsfonts,amscd,amssymb,bm,mathrsfs}
\usepackage{enumerate}
\usepackage{appendix}
\usepackage{graphicx,xcolor}
\usepackage{epstopdf}
\usepackage{subfigure}
\usepackage[normalem]{ulem}
\usepackage{indentfirst}
\usepackage{cases}
\useunder{\uline}{\ul}{}
\DeclareMathOperator*{\argmin}{arg\,min}
\newcommand{\R}{{\mathbb{R}}}

\newcommand{\E}{{\mathbb{E}}}

\newcommand{\norm}[1]{{\left\vert\kern-0.25ex\left\vert\kern-0.25ex\left\vert #1
\right\vert\kern-0.25ex\right\vert\kern-0.25ex\right\vert}}

\newtheorem{lemma}{Lemma}
\newtheorem{definition}{Definition}
\newtheorem{proposition}{Proposition}

\newtheorem{theorem}{Theorem}

\usepackage{lineno}

\journal{xx}
\begin{document}

\begin{frontmatter}



%\title{Adversarially robust generalization theory for deep neural networks with Jacobian regularization} %% Article title

\title{Adversarially robust generalization theory via Jacobian regularization for deep neural networks } %% Article title


%\author{Dongya Wu} %% Author name
%
%%% Author affiliation
%\affiliation{organization={School of Information Science and Technology, Northwest University},%Department and Organization
%            addressline={Xuefu Road}, 
%            city={Xi'an},
%            postcode={710127}, 
%            state={Shaanxi Province},
%            country={P. R. China}}

\author[1]{Dongya Wu\corref{cor1}}
\ead{wudongya@nwu.edu.cn}         
\author[2]{Xin Li\corref{cor2}}
\ead{lixin@nwu.edu.cn}

\address[1]{School of Information Science and Technology, Northwest University, Xi’an, 710069, China}
\address[2]{School of Mathematics, Northwest University, Xi’an, 710069, China}

\cortext[cor1]{Corresponding author}
\cortext[cor2]{Corresponding author}

%% Abstract
\begin{abstract}
Powerful deep neural networks are vulnerable to adversarial attacks. To obtain adversarially robust models, researchers have separately developed adversarial training and Jacobian regularization techniques. There are abundant theoretical and empirical studies for adversarial training, but theoretical foundations for Jacobian regularization are still lacking. In this study, we show that Jacobian regularization is closely related to adversarial training in that $\ell_{2}$ or $\ell_{1}$ Jacobian regularized loss serves as an approximate upper bound on the adversarially robust loss under $\ell_{2}$ or $\ell_{\infty}$ adversarial attack respectively. Further, we establish the robust generalization gap for Jacobian regularized risk minimizer via bounding the Rademacher complexity of both the standard loss function class and Jacobian regularization function class. Our theoretical results indicate that the norms of Jacobian are related to both standard and robust generalization. We also perform experiments on MNIST data classification to demonstrate that Jacobian regularized risk minimization indeed serves as a surrogate for adversarially robust risk minimization, and that reducing the norms of Jacobian can improve both standard and robust generalization. This study promotes both theoretical and empirical understandings to adversarially robust generalization via Jacobian regularization.
\end{abstract}



%% Keywords

\end{frontmatter}

%\linenumbers


\section{Introduction}
Deep neural networks have achieved the state-of-the-art performance in various tasks such as image classification, speech recognition, object recognition, natural language processing and so on \citep{lecun2015deep}. However, deep neural networks are vulnerable to adversarial attacks \citep{Szegedy2013IntriguingPO,Goodfellow2014ExplainingAH} and the generalization performance is dramatically impaired by these adversarial samples \citep{Madry2017TowardsDL}. 

To obtain adversarially robust generalization for deep neural networks, adversarial training \citep{Goodfellow2014ExplainingAH, Huang2015LearningWA, Shaham2015UnderstandingAT, Madry2017TowardsDL, Gowal2020UncoveringTL} has been developed as one of the most effective defense methods against adversarial attacks \citep{Qian2022ASO}. Adversarial training aims to achieve robust generalization on each sample $x$ within a small perturbation $\| x' - x \|_{p} \leq \epsilon$ via the empirical adversarial risk minimization
\begin{equation} \label{arm-objective}
\min_{f \in \mathcal{F}} \frac1n \sum_{i=1}^{n} \max_{\| x'_{i} - x_{i} \|_{p} \leq \epsilon} \ell(f(x'_{i}),y_{i}),
\end{equation}
where $\mathcal{F}$ is the hypothesis class, $\{(x_{i},y_{i})\}_{i=1}^{n}$ are $n$ training examples drawn from distribution $\mathcal{D}$, $\ell(\cdot,\cdot)$ is the loss function. 
%\textcolor{red}{A major challenge is that adversarial training still possesses relatively large robust generalization gap compared with its counterpart in the standard setting \citep{Madry2017TowardsDL, Gowal2020UncoveringTL}.} 
It is important to theoretically characterize the generalization of adversarial training by bounding the robust generalization gap between the expected adversarial risk and the empirical adversarial risk
\begin{equation} \label{rg-gap}
\E_{(x,y) \sim \mathcal{D}} \max_{\| x' - x \|_{p} \leq \epsilon} \ell(f(x'),y) - \frac1n \sum_{i=1}^{n} \max_{\| x'_{i} - x_{i} \|_{p} \leq \epsilon} \ell(f(x'_{i}),y_{i}).
\end{equation}

In the standard setting without adversarial attacks, it is well known in statistical learning theory that the generalization gap can be bounded by the Rademacher complexity \citep{Bartlett2003RademacherAG} and then the Rademacher complexity of deep neural networks has been established \citep{RN481,Neyshabur2017APA,RN488,RN323}. To bound the robust generalization gap in the adversarial setting, \cite{Khim2018AdversarialRB} and \cite{Yin2018RademacherCF} showed that the robust generalization gap can be bounded by the adversarial Rademacher complexity, which is an extending of the Rademacher complexity. However, it is difficult to bound the adversarial Rademacher complexity of deep neural networks due to the $\max$ operation in the adversarially robust loss, i.e., $\max_{\| x' - x \|_{p} \leq \epsilon} \ell(f(x'),y)$. One type of studies designed surrogate losses of the adversarially robust loss and provided bounds on the surrogate losses, such as the tree-transformation surrogate loss \citep{Khim2018AdversarialRB}, the 
SDP relaxation surrogate loss \citep{Yin2018RademacherCF} and the FGSM surrogate loss \citep{Gao2021TheoreticalIO}. Another type of studies bounds the adversarial Rademacher complexity via calculating the covering number of adversarial function classes \citep{Xiao2022AdversarialRC, Mustafa2022OnTG, xiao2024bridging}.

Another way to achieve adversarially robust generalization is via Jacobian regularization \citep{7934087, Varga2017GradientRI, Hoffman2019RobustLW, Jakubovitz2018ImprovingDR, Chan2019JacobianAR, LIU2024109902}. \cite{Hoffman2019RobustLW} showed that Jacobian regularization increases classification margins and outperforms an adversarial training defense to achieve lower adversarially robust generalization error. \cite{LIU2024109902} utilized Jacobian regularization to improve both robust generalization and prediction interpretability in the adversarial setting. There is also an empirical study which showed that the norm of Jacobian correlates well with generalization \citep{Novak2018SensitivityAG}. Even though it is well known that Jacobian regularization can make the model more stable with respect to input perturbations \citep{Hoffman2019RobustLW}, however, there is still no theoretical guarantee for the adversarially robust generalization gap achieved via Jacobian regularization.

The contributions of our study are summarized in the following aspects. Firstly, we show that Jacobian regularization is closely related to adversarial training in the sense that the Jacobian regularized loss serves as an approximate upper bound on the adversarially robust loss, while former researches only study adversarial training and Jacobian regularization separately. Secondly, we provide the first robust generalization bound for deep neural networks via Jacobian regularization. The key ingredients in obtaining the robust generalization bound are the control of the Rademacher complexity of the standard loss function class and the Jacobian regularization function class. Thirdly, unlike most results focusing on the Rademacher complexity of real-valued models, we bound the Rademacher complexity for widely-used vector-valued models, such as the multi-class classification in practice. We do not use the commonly-used vector-contraction inequality to bound the Rademacher complexity of vector-valued models, due to the disadvantage that the result may explicitly depend on the output dimension. Instead, the covering number technique is adopted with no explicit dependence on the output dimension. In addition, the Rademacher complexity of the Jacobian regularization function class is novel to the best of our knowledge and has never been studied before.  Beyond the adversarially robust generalization, this will also be interested in the field where gradient regularization is used, such as nonparametric sparse variable selection \citep{RN278}.

%The contributions of our study are summarized in the following aspects. While adversarial training and Jacobian regularization are developed separately, in this study, we show that Jacobian regularization is closely related to adversarial training in the sense that Jacobian regularized loss can serve as an approximate upper bound on the adversarially robust loss. Furthermore, we provide the first robust generalization bound for deep neural networks with Jacobian regularization. The key ingredients in obtaining the robust generalization bound are the Rademacher complexity of standard loss function class and Jacobian regularization function class. Unlike most Rademacher complexity of standard loss function class which is obtained for real-valued models, we bound the Rademacher complexity for vector-valued models, since vector-valued models are usually adopted in practice, such as the multi-class classification. We do not utilize the commonly used vector-contraction inequality for bounding the Rademacher complexity of vector-valued models, which can result in an explicit dependence on the output dimension. We adopt the covering number technique and there is no explicit dependence on the output dimension. In addition, the Rademacher complexity of Jacobian regularization function class is novel and has never been studied before, which will also be interested beyond the adversarially robust generalization where gradient regularization is also used, such as nonparametric sparse variable selection \citep{RN278}.

\vskip 10pt
\textbf{Notations}. We end this section by introducing some useful notations. For a vector $x=(x^{1},x^{2},\cdots,x^{d})\in \R^{d}$, define the $\ell_{p}$-norm of $x$ as $\|x\|_{p}=(\sum_{j=1}^{d}|x^{i}|^{p})^{1/p}$ for $1\leq p< \infty$ with $\|x\|_{\infty}=\max_{i=1,2,\cdots,d}|x^{i}|$.  For a matrix $A\in \R^{d_{1}\times d_{2}}$, let $A_{ij}\ (i=1,\dots,d_{1},j=1,2,\cdots,d_{2})$ denote its $ij$-th entry, $A_{i\cdot}\ (i=1,\dots,d_{1})$ denote its $i$-th row, $A_{\cdot j}\ (j=1,2,\cdots,d_{2})$ denote its $j$-th column. Let $\|A\|_{F}$ denote the Frobenius norm, $\|A\|_{\text{op}}$ denote the spectral norm, and $\|A\|_{1,1}$ denote the $\ell_1$-norm along all the entries of $A$. 

\section{Related works}

Our study that establishes the robust generalization gap of Jacobian regularization is related to both the adversarial training and Jacobian or gradient regularization.

\textbf{Adversarial training}. Adversarial training is one of the most effective methods to achieve robust generalization and many studies have established the generalization property of adversarial training theoretically. The theoretical studies either adopt surrogate losses of the adversarially robust loss  \citep{Khim2018AdversarialRB, Yin2018RademacherCF, Gao2021TheoreticalIO} or directly bound the adversarial Rademacher complexity via covering number \citep{Xiao2022AdversarialRC, Mustafa2022OnTG, xiao2024bridging}. The loss with Jacobian regularization in our study serves as an upper bound of the first order approximation of adversarially robust loss. Therefore, our study also belongs to the method of surrogate losses of adversarially robust loss. The main difference between our study and previous studies lies in that the theory in our study matches with practice perfectly. Specifically, previous studies still rely on projected gradient descent adversarial training \citep{Madry2017TowardsDL} and thus there is a mismatch between the surrogate loss used in theory and the training objective used in practice. In our study, since Jacobian regularization can be easily optimized, we optimize the surrogate loss with Jacobian regularization to achieve robust generalization, in which way the practice is guaranteed to match with our theoretical analysis on the surrogate loss with Jacobian regularization.

\textbf{Jacobian or gradient regularization}. Though there are a variety of related works which empirically show that Jacobian or gradient regularization can improve models' robustness to adversarial attacks \citep{Varga2017GradientRI, Hoffman2019RobustLW, Jakubovitz2018ImprovingDR, Chan2019JacobianAR, LIU2024109902, Lyu2015AUG, NECO_a_00928, Ross2017ImprovingTA}, theoretical analysis on the generalization is still lacking. We here establish the generalization gap of Jacobian regularization and demonstrate its ability to achieve robust generalization theoretically. The main difference between the Jacobian regularization in \cite{Hoffman2019RobustLW} and that in our study lies in that we propose two Jacobian regularizations $\|\mathcal{J}\|_{F}^{2}$ and $\|\mathcal{J}\|_{1,1}$ against $\ell_{2}$ and $\ell_{\infty}$ attack respectively. And the Jacobian regularization in \cite{Hoffman2019RobustLW} is developed separately with adversarial training, while we link Jacobian regularization to adversarial training via the first order approximation. \cite{Lyu2015AUG} also link the gradient of the loss function to adversarial training via the first order approximation. The main difference is that Jacobian regularization regularize the gradient of the model. If we adopt the cross entropy loss or mean squared loss, the gradient of the loss function is small if the loss is small, but the gradient of the model can still be large. %Thus Jacobian regularization can be more efficient than loss gradient regularization. In addition, the loss gradient regularization needs the information of response variable $y$ and should be performed in the supervised learning, but Jacobian regularization only needs the information of prediction variable $x$ and can be performed in the semi-supervised learning.

\section{Problem setup}

Let $x\in \mathbb{R}^d$ be the input prediction variable and $y\in \R^k$ be the response variable. The pair $(x,y)$ obeys a distribution $\mathcal{D}(x,y)$ on a sample space $\mathcal{X}\times \mathcal{Y}$, and the marginal distribution of $x$ is $\mathcal{D}_x(x)$. To learn the relationship between $x$ and $y$ in the standard setting, the learning problem is to minimize the expected risk $R(f)=\E_{(x,y)\sim\mathcal{D}} [ \ell(f(x),y) ]$, where $\ell(\cdot,\cdot)$ is the loss function and $f \in \mathcal{F}$ with $\mathcal{F}$ being the hypothesis class of deep neural networks in this study. A deep neural network $f_{\Theta}: \R^d \rightarrow \R$ with $L$ layers can be represented as
\begin{equation}\label{eq-DNN}
f_{\Theta}(x)=\theta_L\sigma(\theta_{L-1}\cdots\sigma(\theta_1 x)\cdots),
\end{equation}
where $\theta_l\in \R^{d_l\times d_{l-1}}(l=1,\cdots,L,\ d_0=d,\ d_L=k)$, $\Theta=\{\theta_L,\theta_{L-1},\cdots,\theta_1\}$, and $\sigma(\cdot)$ is the activation function, which is assumed to be 1$\text{-}$Lipschitz and to satisfy $\sigma(0)=0$. Then it follows that $\mathcal{F}=\{f_{\Theta}: f_{\Theta}\ \text{satisfies}\ \eqref{eq-DNN}\}$.
Since the distribution $\mathcal{D}$ is unknown, we can only get a finite dataset $\mathbb{D}$ with $n$ independently and identically distributed samples $\{(x_{i},y_{i})\}_{i=1}^{n}$ drawn from $\mathcal{D}$. We can thus minimize the empirical risk $R_{n}(f)=\frac1n \sum_{i=1}^{n}  \ell(f(x_{i}),y_{i})$. The gap between the expected risk and the empirical risk is known as the generalization gap.

In the adversarial setting, define the adversarially robust loss as 
\begin{equation} \label{robust-loss}
\tilde{\ell}(f(x),y) = \max_{\| x' - x \|_{p} \leq \epsilon} \ell(f(x'),y), 
\end{equation}
which is the maximum loss within a perturbation ball centered at $x$. Define the adversarial expected risk as $\tilde{R}(f)=\E_{(x,y)\sim\mathcal{D}} [ \tilde{\ell}(f(x),y) ]$ and the adversarial empirical risk as $\tilde{R}_{n}(f)=\frac1n \sum_{i=1}^{n}  \tilde{\ell}(f(x_{i}),y_{i})$. To learn models robust to adversarial attacks, adversarial training minimizes the adversarial empirical risk. The gap between the adversarial expected risk and the adversarial empirical risk is known as the robust generalization gap. 

Define the loss function class as $\ell_{\mathcal{F}}=\{\ell(f(x),y) | f \in \mathcal{F}  \}$ and the robust loss function class as $\tilde{\ell}_{\mathcal{F}}=\{\tilde{\ell}(f(x),y) | f \in \mathcal{F}  \}$. Then the generalization gap and robust generalization gap can be bounded by the Rademacher complexity of $\ell_{\mathcal{F}}$ and $\tilde{\ell}_{\mathcal{F}}$ respectively.

\begin{definition}[Rademacher complexity]
Let $\xi_{1},\dots,\xi_{n}$ be $n$ independent Rademacher random variables that take values of $1$ or $-1$ with probability $1/2$. Given a dataset $S=\{x_{1},\dots,x_{n}\}$ with $n$ independent samples drawn from $\mathcal{D}_x(x)$, 
the empirical Rademacher complexity of hypothesis class $\mathcal{F}$ is defined as
\begin{equation*}
\mathcal{R}_{S}(\mathcal{F}) = \E_{\xi}\left[\sup\limits_{f \in \mathcal{F}} \frac{1}{n}\sum\limits_{i=1}^n \xi_{i} f(x_{i})\right].
\end{equation*}

%The Rademacher complexity of hypothesis class $\mathcal{F}$ is defined as
%\begin{equation*}
%\mathcal{R}_{n}(\mathcal{F}) =\E_{S}[\mathcal{R}_{S}(\mathcal{F})]= \E_{S}\E_{\xi}\left[\sup\limits_{f \in \mathcal{F}} \frac{1}{n}\sum\limits_{i=1}^n \xi_{i} f(x_{i})\right].
%\end{equation*}

\end{definition}

\begin{proposition}\emph{\citep{Bartlett2003RademacherAG}}\label{proposition-standard}
Suppose the loss function $\ell(f(x),y)$ is bounded between $[0,B]$. Then for any $\delta \in (0,1)$, with probability at least $1-\delta$, the following inequality holds for all $f \in \mathcal{F}$,
\begin{equation*}
R(f) \leq R_{n}(f) + 2\mathcal{R}_{S}(\ell_{\mathcal{F}}) + 3B\sqrt{\frac{\log \frac{2}{\delta} }{2n}}.
\end{equation*}
\end{proposition}

\begin{proposition}\emph{\citep{Yin2018RademacherCF}}
Suppose the loss function $\tilde{\ell}(f(x),y)$ is bounded between $[0,B]$. Then for any $\delta \in (0,1)$, with probability at least $1-\delta$, the following holds for all $f \in \mathcal{F}$,
\begin{equation*}
\tilde{R}(f) \leq \tilde{R}_{n}(f) + 2\mathcal{R}_{S}(\tilde{\ell}_{\mathcal{F}}) + 3B\sqrt{\frac{\log \frac{2}{\delta} }{2n}}.
\end{equation*}
\end{proposition}

The above bounds hold uniformly for all $f \in \mathcal{F}$, especially for the minimizer of the empirical risk or adversarial empirical risk. It is easy to see that the key ingredient of the generalization gap is the Rademacher complexity. In the standard setting with no adversarial attacks, the Rademacher complexity of the deep neural network has been relatively well developed. While in the adversarial setting, it is difficult to bound the adversarial Rademacher complexity $\mathcal{R}_{S}(\tilde{\ell}_{\mathcal{F}})$ for the deep neural network due to the $\max$ operation in \eqref{robust-loss}. 

\section{Main results}
  
\subsection{Surrogate robust loss with Jacobian regularization}\label{sec-model}
Since solving the adversarially robust loss in \eqref{robust-loss} and bounding the adversarial Rademacher complexity are difficult, we approximate the adversarially robust loss via first order Taylor expansion. Assume that the solution of \eqref{robust-loss} is achieved at $x^{*}$, i.e., $\tilde{\ell}(f(x),y)=\ell(f(x^{*}),y)$, then the adversarially robust loss can be approximated as
\begin{equation}\label{loss-approximation}
\ell(f(x^{*}),y) \approx \ell(f(x),y) + \nabla_{x}\ell(f(x),y)^{\top} (x^{*}-x).
\end{equation}
The following lemma shows the bound on the first order approximation for the $\ell_{2}$ and $\ell_{\infty}$ adversarial attacks, respectively.

\begin{lemma}\label{lemma-approximation}
Assume the loss function $\ell(f(x),y)$ is $L_{\ell_{2}}$-Lipschitz on $f(x)$ with respect to the $\|\cdot\|_{2}$ metric and is $L_{\ell_{\infty}}$-Lipschitz on $f(x)$ with respect to the $\|\cdot\|_{1}$ metric. For the $\ell_{2}$ adversarial attack $\|x'-x\|_{2} \leq \epsilon$, the approximation of adversarially robust loss in \eqref{loss-approximation} can be bounded by 
\begin{equation*}
\ell(f(x),y) + \nabla_{x}\ell(f(x),y)^{\top} (x^{*}-x) \leq \ell(f(x),y) + \frac12\epsilon + \frac12\epsilon L_{\ell_{2}}^{2} \| \nabla_{x}f(x) \|_{F}^{2}.
\end{equation*}
For the $\ell_{\infty}$ adversarial attack $\|x'-x\|_{\infty} \leq \epsilon$, the approximation \eqref{loss-approximation} can be bounded by 
\begin{equation*}
\ell(f(x),y) + \nabla_{x}\ell(f(x),y)^{\top} (x^{*}-x) \leq \ell(f(x),y) + \epsilon L_{\ell_{\infty}} \| \nabla_{x}f(x) \|_{1,1}.
\end{equation*}
\end{lemma}

The term $\frac12\epsilon L_{\ell_{2}}^{2} \| \nabla_{x}f(x) \|_{F}^{2}$ and $\epsilon L_{\ell_{\infty}} \| \nabla_{x}f(x) \|_{1,1}$ serve as the regularization, but since these terms are upper bounds, the regularization may be too strong. Therefore, we add a flexible parameter $\lambda$ on the Jacobian regularization and then define the Jacobian regularized loss which serves as a surrogate robust loss for the $\ell_{2}$ and $\ell_{\infty}$ adversarial attacks respectively as
\begin{equation*}
\hat{\ell}_{2}(f(x),y)=\ell(f(x),y) + \frac12\lambda\epsilon L_{\ell_{2}}^{2} \| \nabla_{x}f(x) \|_{F}^{2},
\end{equation*}
\begin{equation*}
\hat{\ell}_{\infty}(f(x),y)=\ell(f(x),y) + \lambda\epsilon L_{\ell_{\infty}} \| \nabla_{x}f(x) \|_{1,1}.
\end{equation*}
We show that Jacobian regularized loss serves as an approximate upper bound on the adversarially robust loss. Even though the upper bound is not precise because of the first-order approximation and parameter $\lambda$, we will show by experiments that Jacobian regularized loss indeed serves as a surrogate loss for the adversarially robust loss and that Jacobian regularized risk minimization can lead to adversarially robust risk minimization.
The adversarially robust model against $\ell_{2}$ and $\ell_{\infty}$ adversarial attack is obtained by Jacobian regularized risk minimization respectively as
\begin{equation}\label{obj-jacobian}
\hat{f}_{2} = \argmin_{f \in \mathcal{F}} \frac1n \sum_{i=1}^{n} [ \ell(f(x_{i}),y_{i}) + \frac12\lambda\epsilon L_{\ell_{2}}^{2} \| \nabla_{x}f(x_{i}) \|_{F}^{2} ],
\end{equation}
\begin{equation}\label{obj-jacobian2}
\hat{f}_{\infty} = \argmin_{f \in \mathcal{F}} \frac1n \sum_{i=1}^{n} [ \ell(f(x_{i}),y_{i}) + \lambda\epsilon L_{\ell_{\infty}} \| \nabla_{x}f(x_{i}) \|_{1,1} ].
\end{equation}
The above minimization procedure involves a standard empirical risk minimization and a Jacobian regularization term. The regularization coefficient depends on the adversarial attack intensity $\epsilon$ and the Lipschitz constants, which makes sense since one should also regularize with stronger intensity to achieve robustness when the adversarial attack intensity is large. The Lipschitz constants are well bounded for common loss functions and the Lipschitz constants of cross entropy loss function are given by the following Lemma.

\begin{lemma}\label{lemma-lipschitz}
Let $\ell(f(x),y)$ be the cross entropy loss function in classification tasks, then $\ell(f(x),y)$ is $L_{\ell_{2}}$-Lipschitz on $f(x)$ with respect to the $\|\cdot\|_{2}$ metric, with the Lipschitz constant $L_{\ell_{2}} \leq \sqrt{2}$, and $\ell(f(x),y)$ is $L_{\ell_{\infty}}$-Lipschitz on $f(x)$ with respect to the $\|\cdot\|_{1}$ metric, with the Lipschitz constant $L_{\ell_{\infty}} \leq 1$.
\end{lemma}

\subsection{Generalization gap of surrogate robust loss}
Define the empirical mean of the Jacobian norm as 
\begin{equation*}
P_{n}\|\nabla_{x}f\|_{F}^{2}=\frac1n \sum_{i=1}^{n}\|\nabla_{x}f(x_{i})\|_{F}^{2}, \quad  P_{n}\|\nabla_{x}f\|_{1,1}=\frac1n \sum_{i=1}^{n}\|\nabla_{x}f(x_{i})\|_{1,1}.
\end{equation*}
By minimizing the Jacobian regularized empirical risk minimization objective in \eqref{obj-jacobian} and \eqref{obj-jacobian2}, the $P_{n}\|\nabla_{x}\hat{f}_{2}\|_{F}^{2}$ of the minimizer $\hat{f}_{2}$ or the $P_{n}\|\nabla_{x}\hat{f}_{\infty}\|_{1,1}$ of the minimizer $\hat{f}_{\infty}$ can not be large and must be within some bounds. Therefore, for a given dataset $\{(x_{i},y_{i})\}_{i=1}^{n}$, we assume that there exist positive numbers $r_{1}$ and $r_{2}$ such that, in the minimization process of \eqref{obj-jacobian} and \eqref{obj-jacobian2}, the empirical mean of the Jacobian norm $P_{n}\|\nabla_{x}f\|_{F}^{2}\leq r_{2}$ and $P_{n}\|\nabla_{x}f\|_{1,1}\leq r_{1}$. Hence the minimization process of \eqref{obj-jacobian} and \eqref{obj-jacobian2} can be done on the constrained hypothesis class $\mathcal{F}_{r_{2}} = \{f\in\mathcal{F}: P_{n}\|\nabla_{x}f\|_{F}^{2} \leq r_{2}\}$ and $\mathcal{F}_{r_{1}} = \{f\in\mathcal{F}: P_{n}\|\nabla_{x}f\|_{1,1} \leq r_{1}\}$ respectively as follows 
\begin{equation}\label{obj-jacobian-r}
\hat{f}_{2} = \argmin_{f \in \mathcal{F}_{r_{2}}} \frac1n \sum_{i=1}^{n} [ \ell(f(x_{i}),y_{i}) + \frac12\lambda\epsilon L_{\ell}^{2} \| \nabla_{x}f(x_{i}) \|_{F}^{2} ].
\end{equation}
\begin{equation}\label{obj-jacobian-r2}
\hat{f}_{\infty} = \argmin_{f \in \mathcal{F}_{r_{1}}} \frac1n \sum_{i=1}^{n} [ \ell(f(x_{i}),y_{i}) + \lambda\epsilon L_{\ell_{\infty}} \| \nabla_{x}f(x_{i}) \|_{1,1} ].
\end{equation}
The bounds $r_{2}$ or $r_{1}$ and the hypothesis classes $\mathcal{F}_{r_{2}}$ or $\mathcal{F}_{r_{1}}$ depends on the training dataset $\{(x_{i},y_{i})\}_{i=1}^{n}$ and the training algorithm. The hypothesis classes $\mathcal{F}_{r_{2}}$ or $\mathcal{F}_{r_{1}}$ are the effective hypothesis classes explored by the training algorithm and are smaller than the original hypothesis class $\mathcal{F}$ due to Jacobian regularization. Even though these bounds are difficult to be calculated directly in theory, in practice, we can still calculate the empirical mean of the Jacobian norm $P_{n}\|\nabla_{x}f\|_{F}^{2}$ or $P_{n}\|\nabla_{x}f\|_{1,1}$ and show that such bounds indeed exist and are small in the experiments.

For $\ell_{2}$ and $\ell_{\infty}$ adversarial attacks, define the surrogate expected risk respectively as $\hat{R}_{2}(f)=\E_{(x,y)\sim\mathcal{D}} [ \hat{\ell}_{2}(f(x),y) ]$ and $\hat{R}_{\infty}(f)=\E_{(x,y)\sim\mathcal{D}} [ \hat{\ell}_{\infty}(f(x),y) ]$, and the surrogate empirical risk respectively as $\hat{R}_{2n}(f)=\frac1n \sum_{i=1}^{n}  \hat{\ell}_{2}(f(x_{i}),y_{i})$ and $\hat{R}_{\infty n}(f)=\frac1n \sum_{i=1}^{n}  \hat{\ell}_{\infty}(f(x_{i}),y_{i})$.
Define the surrogate loss function class respectively as $\hat{\ell}_{2\mathcal{F}_{r_{2}}}=\{\hat{\ell}_{2}(f(x),y) | f \in \mathcal{F}_{r_{2}}  \}$ and $\hat{\ell}_{\infty \mathcal{F}_{r_{1}}}=\{\hat{\ell}_{\infty}(f(x),y) | f \in \mathcal{F}_{r_{1}}  \}$. We now bound the generalization gap of the surrogate robust loss by the following proposition.
\begin{proposition}\label{proposition-surrogate}
Suppose $\ell(f(x),y)$ is bounded between $[0,B]$, $\| \nabla_{x}f(x) \|_{F}^{2}$ is bounded between $[0,B_{2}]$ and $\| \nabla_{x}f(x) \|_{1,1}$ is bounded between $[0,B_{1}]$. Denote $C_{2}=B+\frac12\lambda\epsilon L_{\ell_{2}}^{2}B_{2}$ and $C_{\infty}=B+\lambda\epsilon L_{\ell_{\infty}}B_{1}$. Then for any $\delta \in (0,1)$, with probability at least $1-\delta$, for the given dataset $\{(x_{i},y_{i})\}_{i=1}^{n}$, the following inequality holds for the minimizer $\hat{f}_{2}$ in \eqref{obj-jacobian-r},
\begin{equation*}
\hat{R}_{2}(\hat{f}_{2}) \leq \hat{R}_{2n}(\hat{f}_{2}) + 2\mathcal{R}_{S}(\hat{\ell}_{2\mathcal{F}_{r_{2}}}) + 3C_{2}\sqrt{\frac{\log \frac{2}{\delta} }{2n}},
\end{equation*}
the following inequality holds for the minimizer $\hat{f}_{\infty}$ in \eqref{obj-jacobian-r2},
\begin{equation*}
\hat{R}_{\infty}(\hat{f}_{\infty}) \leq \hat{R}_{\infty n}(\hat{f}_{\infty}) + 2\mathcal{R}_{S}(\hat{\ell}_{\infty\mathcal{F}_{r_{1}}}) + 3C_{\infty}\sqrt{\frac{\log \frac{2}{\delta} }{2n}}.
\end{equation*}
\end{proposition}

In standard statistical learning theory, the hypothesis class needs to be fixed, but the effective hypothesis class $\mathcal{F}_{r_{2}}$ or $\mathcal{F}_{r_{1}}$ depends on the training dataset. This is not a contradictory since we can conduct the proof as if $\mathcal{F}_{r_{2}}$ or $\mathcal{F}_{r_{1}}$ is fixed for the given dataset $\{(x_{i},y_{i})\}_{i=1}^{n}$, or we can choose the largest $r_{2}$ or $r_{1}$ when the training dataset varies. 
The key of the generalization gap is the Rademacher complexity of the surrogate loss $\mathcal{R}_{S}(\hat{\ell}_{2\mathcal{F}_{r_{2}}})$ or $\mathcal{R}_{S}(\hat{\ell}_{\infty \mathcal{F}_{r_{1}}})$ which are shown by the following lemma.

\begin{lemma}\label{lemma-decomposition}
For the $\ell_{2}$ and $\ell_{\infty}$ attacks, denote the function class of Jacobian regularization respectively as $\mathcal{J}_{2\mathcal{F}_{r_{2}}}=\{ \| \nabla_{x}f(x) \|_{F}^{2} : f \in \mathcal{F}_{r_{2}} \}$ and $\mathcal{J}_{1\mathcal{F}_{r_{1}}}=\{ \| \nabla_{x}f(x) \|_{1,1} : f \in \mathcal{F}_{r_{1}} \}$, then the Rademacher complexity of the surrogate loss $\mathcal{R}_{S}(\hat{\ell}_{2\mathcal{F}_{r_{2}}})$ or $\mathcal{R}_{S}(\hat{\ell}_{\infty \mathcal{F}_{r_{1}}})$ can be bounded respectively by 
\begin{equation*}
\mathcal{R}_{S}(\hat{\ell}_{2\mathcal{F}_{r_{2}}}) \leq  \mathcal{R}_{S}(\ell_{\mathcal{F}_{r_{2}}})+\frac12\lambda\epsilon L_{\ell_{2}}^{2} \mathcal{R}_{S}(\mathcal{J}_{2\mathcal{F}_{r_{2}}}),
\end{equation*}
\begin{equation*}
\mathcal{R}_{S}(\hat{\ell}_{\infty\mathcal{F}_{r_{1}}}) \leq  \mathcal{R}_{S}(\ell_{\mathcal{F}_{r_{1}}})+\lambda\epsilon L_{\ell_{\infty}} \mathcal{R}_{S}(\mathcal{J}_{1 \mathcal{F}_{r_{1}}}).
\end{equation*}
\end{lemma}

This proposition tells us that the Rademacher complexity of the surrogate loss can be upper bounded by the standard Rademacher complexity $\mathcal{R}_{S}(\ell_{\mathcal{F}_{r_{2}}})$ or $\mathcal{R}_{S}(\ell_{\mathcal{F}_{r_{1}}})$ and the Rademacher complexity $\mathcal{R}_{S}(\mathcal{J}_{2\mathcal{F}_{r_{2}}})$ or $\mathcal{R}_{S}(\mathcal{J}_{1 \mathcal{F}_{r_{1}}})$ induced by Jacobian regularization, which suggests that the robust generalization gap is larger than the standard generalization gap given the same number of training samples. This result also coincides with the study which showed that adversarially robust generalization requires more data \citep{Schmidt2018AdversariallyRG}. 

We next bound the Rademacher complexity via the covering number technique. First recall the definition of the covering number.

\begin{definition}[Covering number]
The $\delta$-covering number of set $\mathcal{Q}$ with respect to metric $\rho$ is defined as the minimum size of $\delta$-cover $\mathcal{C}$ of $\mathcal{Q}$, such that for each $v \in \mathcal{Q}$, there exists $v' \in \mathcal{C}$ satisfying $\rho(v,v') \leq \delta$:
\begin{equation*}
\mathcal{N}(\delta, \mathcal{Q}, \rho) = \inf \{ |\mathcal{C}|: \mathcal{C} \text{ is a }  \delta \text{-cover of } \mathcal{Q} \text{ with respect to metric } \rho \}.
\end{equation*}
\end{definition}
Then some norms and metrics on the function classes and the parameter space are given below. 

%For the loss function class $\ell_{\mathcal{F}_{r_{2}}}$ or $\ell_{\mathcal{F}_{r_{1}}}$, we define the sample $L_{2}(P_{n})$-norm of a function $\ell(f(x),y)$ and the derived metric respectively as
%\begin{equation*}
%\| \ell \|_{L_{2}(P_{n})} = \sqrt{ \frac{1}{n} \sum_{i=1}^{n} (\ell(f(x_{i}),y_{i}))^{2} }\quad \text{and}\quad \rho(\ell,\ell') = \| \ell-\ell' \|_{L_{2}(P_{n})}.
%\end{equation*}

For the loss function class $\ell_{\mathcal{F}_{r_{2}}}$ or $\ell_{\mathcal{F}_{r_{1}}}$, we define the sample $L_{1}(P_{n})$-norm of a function $\ell(f(x),y)$ and the derived metric respectively as
\begin{equation*}
\| \ell \|_{L_{1}(P_{n})} =  \frac{1}{n} \sum_{i=1}^{n}  | \ell(f(x_{i}),y_{i}) | \quad \text{and}\quad \rho(\ell,\ell') = \| \ell-\ell' \|_{L_{1}(P_{n})}.
\end{equation*}
For the Jacobian regularization function class $\mathcal{J}_{2\mathcal{F}_{r_{2}}}$, we define the sample $L_{1}(P_{n})$-norm of a function $\|\nabla_{x}f(x)\|^{2}_{F}$ and the derived metric respectively as
\begin{equation*}
\begin{aligned}
 \| \|\nabla_{x}f\|^{2}_{F} \|_{L_{1}(P_{n})} &=  \frac{1}{n} \sum_{i=1}^{n} \left| \|\nabla_{x}f(x)\|^{2}_{F} \right| \quad \text{and} \\
 \rho(\|\nabla_{x}f\|^{2}_{F},\|\nabla_{x}f'\|^{2}_{F}) &= \| \|\nabla_{x}f\|^{2}_{F} - \|\nabla_{x}f'\|^{2}_{F} \|_{L_{1}(P_{n})}.
\end{aligned}
\end{equation*}
For the Jacobian regularization function class $\mathcal{J}_{1\mathcal{F}_{r_{1}}}$, the $L_{1}(P_{n})$-norm of a function $\|\nabla_{x}f(x)\|_{1,1}$ is defined in the same way with $\mathcal{J}_{2\mathcal{F}_{r_{2}}}$.
For the parameter set $\varTheta$, define the Frobenius-norm of parameter $\Theta \in \varTheta$ and the derived metric respectively as
\begin{equation*}
\norm{\Theta}_{F} = \sqrt{\sum_{l=1}^{L} \| \theta_{l} \|_{F}^{2}} = \sqrt{\sum_{l=1}^{L}\sum_{k=1}^{d_{l}}\sum_{j=1}^{d_{l-1}} (\theta_{l})_{kj}^{2}}\quad \text{and}\quad  \rho(\Theta,\Theta') = \norm{ \Theta-\Theta' }_{F}. 
\end{equation*}

To link the covering number of the function space with the covering number of the parameter space, we need to characterize the Lipschitz property with respect to parameters by the following Lemmas.

\begin{lemma}\label{lipschitz-parameter}
Assume that the loss function $\ell(f(x),y)$ is $L_{\ell_{2}}$-Lipschitz on $f(x)$ with respect to the $\|\cdot\|_{2}$ metric, the Frobenius-norm of parameter $\Theta$ is within $R_{\Theta}$, i.e., $\norm{\Theta}_{F} \leq R_{\Theta}$, and the activation function $\sigma(\cdot)$ is the rectified linear unit function $\emph{Relu}(x)=\max\{x,0\}$. For $\Theta, \Theta' \in \varTheta$, it follows that

\begin{equation*}
\begin{aligned}
%\| \ell(f_{\Theta}(x),y) - \ell(f_{\Theta'}(x),y) \|_{L_{2}(P_{n})} &\leq L_{2P_{n}} \norm{\Theta-\Theta'}_{F}, \\
\| \ell(f_{\Theta}(x),y) - \ell(f_{\Theta'}(x),y) \|_{L_{1}(P_{n})} &\leq L^{\ell}_{1P_{n}} \norm{\Theta-\Theta'}_{F},
\end{aligned}
\end{equation*}

where
%\begin{equation}
\begin{align}
%L_{\Theta}(x) &= \sqrt{L} L_{\ell_{2}} \max_{l \in \{1,\dots,L \}} \prod_{k=1,k \neq l}^{L} \| \theta_{k} \|_{F} \| x\|_{2},\notag \\
%L_{2P_{n}} &= \sqrt{L} L_{\ell_{2}} \sqrt{ \frac{1}{n} \sum_{i=1}^{n} \| x_{i} \|_{2}^{2} } \left( \frac{R_{\Theta}}{\sqrt{L-1}} \right)^{L-1} .\label{eq-LipL2}
L^{\ell}_{1P_{n}} &= \sqrt{L} L_{\ell_{2}} \frac{1}{n} \sum_{i=1}^{n} \| x_{i} \|_{2} \left( \frac{R_{\Theta}}{\sqrt{L-1}} \right)^{L-1} .\label{eq-LipL2}
\end{align}
%\end{equation}
\end{lemma}

\begin{lemma}\label{lipschitz-jacobian}
Assume that there are no pre-activations $h_{l_{pre}} = (\theta_{l}\cdots\sigma(\theta_1 x)\cdots)$ which are exactly 0, the Frobenius-norm of parameter $\Theta$ is within $R_{\Theta}$, i.e., $\norm{\Theta}_{F} \leq R_{\Theta}$, and the activation function $\sigma(\cdot)$ is the rectified linear unit function $\emph{Relu}(x)=\max\{x,0\}$. For $\Theta, \Theta' \in \varTheta$, it follows that

\begin{equation*}
\begin{aligned}
\| \|\nabla_{x}f_{\Theta}(x)\|^{2}_{F} - \|\nabla_{x}f_{\Theta'}(x)\|^{2}_{F} \|_{L_{1}(P_{n})} &\leq L^{F}_{1P_{n}} \norm{\Theta-\Theta'}_{F} \\
\| \|\nabla_{x}f_{\Theta}(x)\|_{1,1} - \|\nabla_{x}f_{\Theta'}(x)\|_{1,1} \|_{L_{1}(P_{n})} &\leq L^{1}_{1P_{n}} \norm{\Theta-\Theta'}_{F} .
\end{aligned}
\end{equation*}

where
\begin{equation}\label{eq-LipL2}
L^{F}_{1P_{n}} = 2\sqrt{L} \left(\frac{R_{\Theta}}{\sqrt{L-1}}\right)^{2L-1}, \quad L^{1}_{1P_{n}}=\sqrt{Lkd} \left(\frac{R_{\Theta}}{\sqrt{L-1}}\right)^{L-1}
\end{equation}
\end{lemma}
There are some differences lying in the Lipschitz constants of the standard loss function and the Jacobian regularization. Specifically, $L^{\ell}_{1P_{n}}$ depends on the data term $1/n \sum_{i=1}^{n} \| x_{i} \|_{2}$ but $L^{F}_{1P_{n}}$ does not. If we assume that $x$ obeys normal distribution, then $1/n \sum_{i=1}^{n} \| x_{i} \|_{2}$ scales with $\sqrt{d}$, where $d$ is the input dimension. The term $( R_{\Theta}/\sqrt{L-1})^{L-1}$ in $L^{\ell}_{1P_{n}}$ and $L^{1}_{1P_{n}}$shows an exponential dependence on the number of layers $L$ and is due to the worst case Lipschitz property of deep neural networks. The term $( R_{\Theta}/\sqrt{L-1})^{2L-1}$ in $L^{F}_{1P_{n}}$ suggests that the Jacobian regularization $\|\nabla_{x}f(x)\|^{2}_{F}$ behaves like a neural network with $2L$ layers. Note that $L^{1}_{1P_{n}}$ also explicitly depends on the output dimension $k$. Overall, compared with the Lipschitz constant of standard loss function, the Lipschitz constant of Jacobian regularization is also well controlled.

Before we bound the Rademacher complexity, we introduce some useful lemmas to bound the norm of model $f(x)$ and the norm of loss function $\ell(f(x),y)$ via Jacobian norms.

\begin{lemma}\label{lemma-derivative}
Let the activation function $\sigma(\cdot)$ be the rectified linear unit function $\emph{Relu}(x)=\max\{x,0\}$. Then for $x \in \mathcal{X}$ and $f \in \mathcal{F}$, it follows that
\begin{equation*}
f(x) = \nabla_{x}f(x)^{ \top} x.
\end{equation*}
\end{lemma}

In fact, this property is not restricted to the $\text{Relu}(\cdot)$ activation function. Lemma \ref{lemma-derivative} also holds when adopting activation functions such as LeakyRelu.

\begin{lemma}\label{lemma-variance}
Assume that the input space $\mathcal{X}$ is bounded such that the $\ell_{\infty}$-norm of each $x \in \mathcal{X}$ is within $R_{x}$, i.e., $\sup \{\| x \|_{\infty}: x \in \mathcal{X}\}\leq R_{x}$. For $f \in \mathcal{F}_{r_{2}}$ such that $P_{n}\|\nabla_{x}f\|_{F}^{2} \leq r_{2}$, it follows that
\begin{equation*}
\begin{aligned}
P_{n}\|f\|_{2} \leq \sqrt{r_{2}} \sqrt{\frac1n \sum_{i=1}^{n} \| x_{i} \|_{2}^{2} }.   
\end{aligned}
\end{equation*}
For $f \in \mathcal{F}_{r_{1}}$ such that $P_{n}\|\nabla_{x}f\|_{1,1} \leq r_{1}$, it follows that
\begin{equation*}
\begin{aligned}
P_{n}\|f\|_{1} \leq r_{1}R_{x}.   
\end{aligned}
\end{equation*}
\end{lemma}

\begin{lemma}\label{lemma-variance-loss}
Assume that $\sup \{\| x \|_{\infty}: x \in \mathcal{X}\}\leq R_{x}$, and that the loss function $\ell(f(x),y)$ is $L_{\ell_{2}}$-Lipschitz on $f(x)$ with respect to $\|\cdot\|_{2}$ metric and is $L_{\ell_{\infty}}$-Lipschitz on $f(x)$ with respect to $\|\cdot\|_{1}$ metric. Denote $\Theta_{0} \in \varTheta$ to be the parameters being zeros. For $f \in \mathcal{F}_{r_{2}}$, it follows that
\begin{equation*}
\begin{aligned}
\| \ell(f(x),y) - \ell(f_{\Theta_{0}}(x),y) \|_{L_{1}(P_{n})} \leq L_{\ell_{2}}  \sqrt{r_{2}} \sqrt{\frac1n \sum_{i=1}^{n} \| x_{i} \|_{2}^{2} }.
\end{aligned}
\end{equation*}
For $f \in \mathcal{F}_{r_{1}}$, it follows that
\begin{equation*}
\begin{aligned}
\| \ell(f(x),y) - \ell(f_{\Theta_{0}}(x),y) \|_{L_{1}(P_{n})} \leq L_{\ell_{\infty}} r_{1}R_{x}.
\end{aligned}
\end{equation*}
\end{lemma}

With the above lemmas at hand, we are now ready to bound the Rademacher complexity by the following two theorems.

\begin{theorem}\label{rademacher}
Let the Frobenius-norm of parameter $\Theta$ is within $R_{\Theta}$, i.e., $\norm{\Theta}_{F} \leq R_{\Theta}$. Assume that the input space is bounded such that $\sup \{\| x \|_{\infty}: x \in \mathcal{X}\}\leq R_{x}$, and that the loss function $\ell(f(x),y)$ is the cross entropy loss in classification tasks. Denote the total number of parameters as $P$. Then the empirical Rademacher complexity of $\ell_{\mathcal{F}_{r_{2}}}$ and $\ell_{\mathcal{F}_{r_{1}}}$ respectively satisfies that
\begin{equation*}
\begin{aligned}
\mathcal{R}_{S}(\ell_{\mathcal{F}_{r_{2}}}) \leq 12 \sqrt{\frac{2r_{2}}{n} \sum_{i=1}^{n} \| x_{i} \|_{2}^{2} } \sqrt{\frac{P}{n}} \left( \sqrt{\log \left| \frac{3R_{\Theta}L^{\ell}_{1P_{n}}}{\sqrt{2r_{2}/n \sum_{i=1}^{n} \| x_{i} \|_{2}^{2} }} \right| } +\sqrt{ \frac{\pi}{2} } \right),
\end{aligned}
\end{equation*}
\begin{equation*}
\begin{aligned}
\mathcal{R}_{S}(\ell_{\mathcal{F}_{r_{1}}}) \leq 12r_{1}R_{x} \sqrt{\frac{P}{n}} \left( \sqrt{\log \left| \frac{3R_{\Theta}L^{\ell}_{1P_{n}}}{r_{1}R_{x}} \right| } +\sqrt{ \frac{\pi}{2} } \right).
\end{aligned}
\end{equation*}
\end{theorem}

One sees from Theorem \ref{rademacher} that the Rademacher complexity of the standard loss function class $\mathcal{R}_{S}(\ell_{\mathcal{F}_{r_{1}}})$ and $\mathcal{R}_{S}(\ell_{\mathcal{F}_{r_{2}}})$ depend on the Jacobian norms $r_{1}$ and $r_{2}$, respectively. By restricting the Jacobian norms $r_{1}$ and $r_{2}$ through Jacobian regularization, one can effectively reduce the Rademacher complexity of the standard loss function class, and further reduce the generalization error in the standard setting without adversarial attacks. In addition, we focus on vector-valued models. Vector-contraction inequality is commonly used for transforming the Rademacher complexity of loss function class to the Rademacher complexity of vector-valued models \citep{Maurer2016AVI,LI2023109356}. However, such technique results in an explicit linear dependence on the output dimension $k$. In Theorem \ref{rademacher}, we directly bound the Rademacher complexity of loss function class via the covering number technique and there is no explicit dependence on the output dimension. This result is more preferable when the output dimension is large.

\begin{theorem}\label{rademacher-jacobian}
Let the Frobenius-norm of parameter $\Theta$ is within $R_{\Theta}$, i.e., $\norm{\Theta}_{F} \leq R_{\Theta}$. Assume that there are no pre-activations which are exactly 0. Denote the total number of parameters as $P$. Then the empirical Rademacher complexity of $\mathcal{J}_{2\mathcal{F}_{r_{2}}}$ and $\mathcal{J}_{1\mathcal{F}_{r_{1}}}$ respectively satisfies that
\begin{equation*}
\begin{aligned}
\mathcal{R}_{S}(\mathcal{J}_{2\mathcal{F}_{r_{2}}}) &\leq 12r_{2} \sqrt{ \frac{P}{n} } \left(\sqrt{ \left| \log  \frac{3R_{\Theta}L^{F}_{1P_{n}}}{r_{2}} \right| }+  \sqrt{\frac{\pi}{2} } \right), \\
\mathcal{R}_{S}(\mathcal{J}_{1\mathcal{F}_{r_{1}}}) &\leq 12r_{1} \sqrt{ \frac{P}{n} } \left(\sqrt{ \left| \log  \frac{3R_{\Theta}L^{1}_{1P_{n}}}{r_{1}} \right| }+  \sqrt{\frac{\pi}{2} } \right).
\end{aligned}
\end{equation*}
\end{theorem}
Due to the fact that Jacobian regularization depends on the training samples, it is natural to ask whether Jacobian regularization is still effective beyond training samples. Comparing the Rademacher complexity $\mathcal{R}_{S}(\ell_{\mathcal{F}_{r_{2}}})$ or $\mathcal{R}_{S}(\ell_{\mathcal{F}_{r_{1}}})$ with $\mathcal{R}_{S}(\mathcal{J}_{2\mathcal{F}_{r_{2}}})$ or $\mathcal{R}_{S}(\mathcal{J}_{1\mathcal{F}_{r_{1}}})$ respectively, one can see that the Rademacher complexity of Jacobian regularization is well controlled relative to the Rademacher complexity of standard loss. Therefore, even though Jacobian regularization relies on sample points, it is still an effective regularization technique. 

Combing the Rademacher complexity of the standard loss function in Theorem \ref{rademacher} and the Rademacher complexity of Jacobian regularization in Theorem \ref{rademacher-jacobian}, we obtain the Rademacher complexity of surrogate robust loss function via Lemma \ref{lemma-decomposition}, and further  establish the robust generalization gap by virtue of Proposition \ref{proposition-surrogate}. If the adversarial attack intensity $\epsilon$ is not large and the Jacobian  norms $r_{1}$ or $r_{2}$ is well controlled, then the robust generalization gap can be well controlled. We will show in the following experiments that reducing Jacobian  norms $r_{1}$ or $r_{2}$ can indeed improve robust generalization.

\section{Experiments}
In this section, we perform experiments to demonstrate that Jacobian regularization can lead to robust generalization. As have been discussed before, the Jacobian norms $r_{1}$ and $r_{2}$ are algorithm dependent, and we will show that $r_{1}$ and $r_{2}$ can be well controlled via Jacobian regularization in the experiments. 
%In addition, the Frobenius-norm $\norm {\Theta}_{F}$ is also algorithm dependent. Jacobian regularization can decrease $R_{\Theta}$, thus decrease the Rademacher complexity and generalization gap. In the experiment, we will also show that $R_{\Theta}$ can be better controlled when adopting Jacobian regularization.

\textbf{Dataset:} We perform the experiments using the MNIST dataset. The MNIST dataset originally contains 60000 training and 10000 testing samples. We randomly select 1000 training samples for training and adopt all the testing samples for testing. We preprocess the data by normalizing the data between 0 and 1.

\textbf{Neural network architecture:} We adopt the multiple-layer fully connected neural network. The number of layers is 5. The number of hidden units in each layer is 100. We use Relu as the nonlinear activation function.

\textbf{Adversarial attacks:} To evaluate the robustness to adversarial attacks, we adopt the %fast gradient sign method (FGSM) and 
projected gradient
descent (PGD) \citep{Kurakin2016AdversarialEI,Madry2017TowardsDL} method to generate adversarial samples. For the $\ell_{2}$ PGD attack, we set the attack intensity $\epsilon$ as 0.5, the number of steps as 20, and the stepsize as 0.1. For the $\ell_{\infty}$ PGD attack, we set the attach intensity $\epsilon$ as 0.03, the number of steps as 20, and the stepsize as 0.01. 

\textbf{Training details:} For the classification task on MNIST, we use the cross entropy function as the loss function. For the $\ell_{2}$ PGD attack, we use the Frobenius norm of the Jacobian matrix for regularization. For the $\ell_{\infty}$ PGD attack, we use the $\ell_{1}$ norm of the Jacobian matrix for regularization. We minimize the regularized cross entropy loss function via SGD. We perform SGD for 1000 epochs, with a batch size of 1000, a momentum of 0.9 and a learning rate of 0.1. 

\begin{figure}[htbp]
\centering
\subfigure[]{
\includegraphics[width=0.48\textwidth]{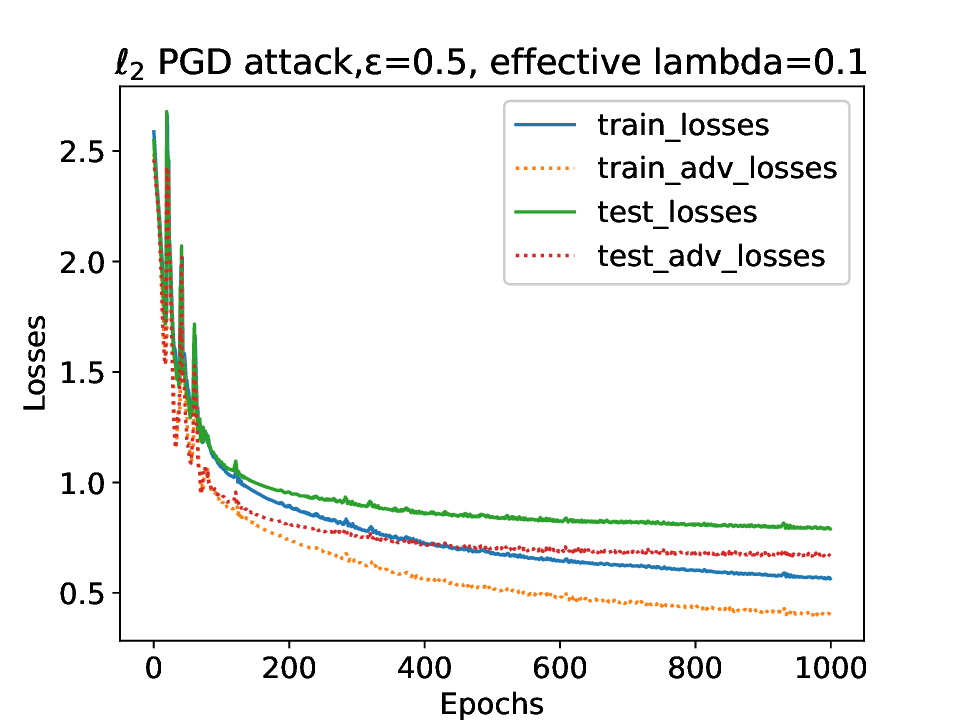}}
\subfigure[]{
\includegraphics[width=0.48\textwidth]{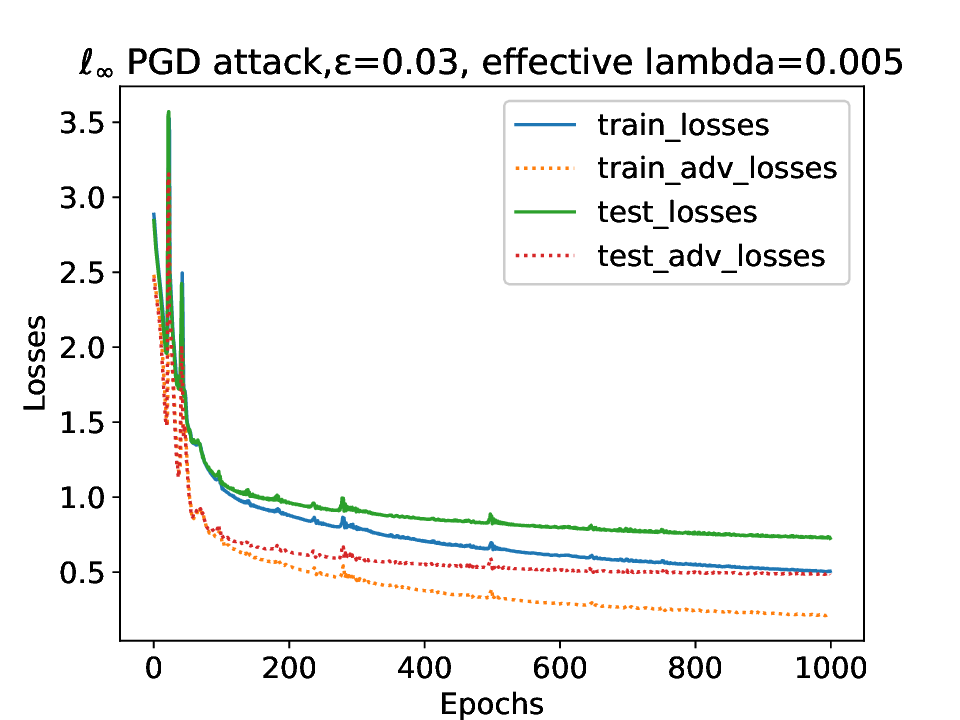}}
\caption{Comparison of Jacobian regularized loss and adversarially robust loss. We use the PGD attack loss to estimate the adversarially robust loss. In the figure legend, the losses represents the Jacobian regularized loss and the adv losses represents the adversarially robust loss.}
\label{comp}
\end{figure}

\textbf{Experimental results:} For simplicity, we define the effective regularization parameter $\tilde{\lambda} = \lambda \epsilon$ (effective lambda for short) which represents the effective Jacobian regularization strength on Jacobian norms. In the theoretical analysis, we show that Jacobian regularized loss serves as an approximate upper bound on the adversarially robust loss. Now we conduct experiments to corroborate this theoretical analysis. We plot the curves of both Jacobian regularized loss and adversarially robust loss (estimated by the PGD attack loss) in Figure \ref{comp}. On both the training and testing data, the adversarially robust losses are well bounded by the Jacobian regularized losses during the training process. Moreover, as we minimize the Jacobian regularized losses during training, the adversarially robust losses also decrease. Therefore, we demonstrate that Jacobian regularized loss indeed serves as a surrogate adversarially robust loss and that Jacobian regularized risk minimization can lead to adversarially robust risk minimization. Note that we only show the result for one choice of effective lambda, but similar results can also be obtained for other choices. In practice, the attack intensity $\epsilon$ is usually unknown and $\lambda$ should be chosen via an additional validation dataset.

More experimental results for $\ell_{2}$ PGD attack are shown in Table \ref{table-l2}, and the experimental results for $\ell_{\infty}$ PGD attack are shown in Table \ref{table-linf}.  
When the effective regularization parameter equals 0, it is just the standard empirical risk minimization without regularizations. When we increase the effective regularization parameter, the Jacobian norms $\|\mathcal{J}\|_{F}^{2}$ and $\|\mathcal{J}\|_{1,1}$ can be controlled within relatively small values. Therefore, the bounds $r_{1}$ and $r_{2}$ used in previous theoretical results can indeed be well controlled under Jacobian regularization, a fact which further implies our assumption made on $r_{1}$ and $r_{2}$ is rational.

\begin{table}[htbp] \small
\begin{tabular}{ccccccc}
\hline
 Effective Lambda & \multicolumn{2}{c}{0}   & \multicolumn{2}{c}{0.01} & \multicolumn{2}{c}{0.1} \\ %\hline
Jacobian Norm           & \multicolumn{2}{c}{1280} & \multicolumn{2}{c}{10.3}   & \multicolumn{2}{c}{3.1}   \\ \hline
Types  & Standard    & Robust    & Standard     & Robust    & Standard     & Robust    \\ %\hline
Training        & 100\%       & 81.8\%      & 100\%        & 100\%     & 99.4\%        & 98.9\%     \\ %\hline
Testing         & 89.5\%        & 69.2\%      & 93.3\%         & 86.7\%      & 93.6\%         & 87.6\%      \\ %\hline
Gap      & 10.5\%            & 12.6\%          &  6.7\%            &  13.3\%         & 5.8\%             & 11.3\%          \\ \hline
\end{tabular}
\caption{Experimental results for $\ell_{2}$ PGD attack. The $\ell_{2}$ attack intensity $\epsilon$ is 0.5.}
\label{table-l2}
\end{table}

\begin{table}[htbp] \small
\begin{tabular}{ccccccc}
\hline
Effective Lambda   & \multicolumn{2}{c}{0}   & \multicolumn{2}{c}{0.001} & \multicolumn{2}{c}{0.005} \\ %\hline
Jacobian Norm           & \multicolumn{2}{c}{2289} & \multicolumn{2}{c}{138.9}   & \multicolumn{2}{c}{76.7}   \\ \hline
Types  & Standard    & Robust    & Standard     & Robust    & Standard     & Robust    \\ %\hline
Training        & 100\%       & 84.7\%      & 100\%        & 100\%     & 100\%        & 99.7\%     \\ %\hline
Testing         & 89.5\%        & 70.3\%      & 92.7\%         & 87.2\%      & 92.8\%         & 88.7\%      \\ %\hline
Gap      & 10.5\%            & 14.4\%          & 7.3\%             & 12.8\%          & 7.2\%          &11.0\%           \\ \hline
\end{tabular}
\caption{Experimental results for $\ell_{\infty}$ PGD attack. The $\ell_{\infty}$ attack intensity $\epsilon$ is 0.03.}
\label{table-linf}
\end{table}

From Table \ref{table-l2} and Table \ref{table-linf}, overall, Jacobian regularization can indeed decrease generalization gap and improve robust generalization against both $\ell_{2}$ and $\ell_{\infty}$ attacks. In Table \ref{table-l2}, the robust generalization gap for standard training with effective regularization parameter being 0 is 12.6\%, even a bit lower than that of effective regularization parameter being 0.01. This is not contradictory because in the standard training, the robust training accuracy can be low. In addition, as pointed by Theorem \ref{rademacher}, we can also observe from the results that reducing the Jacobian norms can also improve the standard generalization. In the standard training, the Jacobian norms can be very large while the standard generalization is not so bad, this phenomenon is because that models unstable to perturbation can also generalize in the standard setting.  

Robust generalization also has relations with standard generalization. Even in the standard setting, the data contain random noise. Improving robustness to adversarial attacks via Jacobian regularization can make model more resistant to the random noise in the standard setting. In general, the random noise in standard setting is less than the adversarial perturbations, and thus Jacobian regularization is more effective for the robust generalization. This phenomenon can also be observed from the results that the standard generalization barely increases while the robust generalization still increases when we increase the effective regularization parameter.

%\begin{figure}[htbp]
%\includegraphics[width=\textwidth]{Fig1.eps}
%\caption{$L^{2}$ error of model prediction and derivative estimation.}
%\label{fig1}
%\end{figure}

\section{Conclusion}
In this study, we demonstrate that standard empirical risk minimization with Jacobian regularization can serve as a surrogate for robust adversarial training. We show that Jacobian regularization can improve both standard generalization and robust generalization from theoretical and empirical aspects. 
Our results can motivate both the theoretical and empirical research to understand the adversarially robust generalization.

%There are also some future works to be considered. 

\section*{Acknowledgments}

Dongya Wu'work was supported in part by he National Natural Science Foundation of China (62103329). Xin Li's work was supported in part by the National Natural Science Foundation of China (12201496).

\appendix
%\section{Proof of Section \ref{sec-model}}\label{sec-appA}
\section{Proofs}

\begin{proof}[Proof of Lemma \ref{lemma-approximation}]
For the $\ell_{2}$ adversarial attack $\|x'-x\|_{2} \leq \epsilon$,
\begin{equation*}
\begin{aligned}
\ell(f(x),y) + \nabla_{x}\ell(f(x),y)^{\top} (x^{*}-x) &\leq \ell(f(x),y) + \| \nabla_{x}\ell(f(x),y) \|_{2} \|x^{*}-x\|_{2} \\
&\leq \ell(f(x),y) + \epsilon \| \nabla_{x}\ell(f(x),y) \|_{2}  \\
&= \ell(f(x),y) + \epsilon \| \nabla_{x}f(x) \nabla_{f}\ell(f(x),y) \|_{2} \\
&\leq \ell(f(x),y) + \epsilon \| \nabla_{x}f(x) \|_{\text{op}} \| \nabla_{f}\ell(f(x),y) \|_{2} \\
&\leq \ell(f(x),y) + \epsilon L_{\ell_{2}} \| \nabla_{x}f(x) \|_{F} \\
&\leq \ell(f(x),y) + \frac12\epsilon + \frac12\epsilon L_{\ell_{2}}^{2} \| \nabla_{x}f(x) \|_{F}^{2}.
\end{aligned}
\end{equation*}
For the $\ell_{\infty}$ adversarial attack $\|x'-x\|_{\infty} \leq \epsilon$,
\begin{equation*}
\begin{aligned}
\ell(f(x),y) + \nabla_{x}\ell(f(x),y)^{\top} (x^{*}-x) &\leq \ell(f(x),y) + \| \nabla_{x}\ell(f(x),y) \|_{1} \|x^{*}-x\|_{\infty} \\
&\leq  \ell(f(x),y) +\epsilon \| \nabla_{x}\ell(f(x),y) \|_{1} \\
&= \ell(f(x),y) + \epsilon \| \nabla_{x}f(x) \nabla_{f}\ell(f(x),y) \|_{1} \\
&\leq \ell(f(x),y) + \epsilon \| \nabla_{x}f(x) \|_{1,1} \| \nabla_{f}\ell(f(x),y) \|_{\infty} \\
&\leq \ell(f(x),y) + \epsilon L_{\ell_{\infty}} \| \nabla_{x}f(x) \|_{1,1}.
\end{aligned}
\end{equation*}
The first inequality follows from Cauchy-Schwartz inequality for $\ell_{2}$ adversarial attack. For $\ell_{\infty}$ adversarial attack, we use the H$\ddot{o}$lder's inequality with $p=\infty, q=1$.
We also bound the spectral norm of Jacobian matrix by the Frobenius norm of Jacobian matrix, since the spectral norm involves a computationally expensive singular value decomposition when the dimension $d$ and $k$ are high. 
\end{proof}

\begin{proof}[Proof of Lemma \ref{lemma-lipschitz}]
Since the indicator function $\mathbbm{1}(y=j)$ equals 1 when the sample belongs to the $j$-th class and otherwise equals 0, and recall that $f_{j}$ is the $j$-th component of $f(x)$, the cross entropy loss function $\ell(f(x),y)$ is defined as
\begin{equation*}
\ell(f(x),y)=\sum_{i=1}^{k}\mathbbm{1}(y=i) \log \frac{\exp(f_{i})}{\sum_{j=1}^{k}\exp(f_{j})}.
\end{equation*}
Denote $\partial \ell / \partial f = [\partial \ell / \partial f_{1},\cdots,\partial \ell / \partial f_{k}]$, then it follows that $L_{\ell_{2}} = \sup \{ \| \partial \ell / \partial f \|_{2} \}$. Denote $\sum_{j}$ to be the sum from indexes 1 to $k$ and $\sum_{j\neq i}$ to be the sum from indexes 1 to $k$ without $i$. One has that the gradient of $\ell(f(x),y)$ with respect to $f_{i}$ is
\begin{equation*}
\begin{aligned}
\partial \ell / \partial f_{i} = \mathbbm{1}(y=i)\frac{\sum_{j\neq i}\exp(f_{j})}{\sum_j\exp(f_{j})} + \sum_{l \neq i} \mathbbm{1}(y=l)\frac{-\exp(f_{i})}{\sum_j\exp(f_{j})}.
\end{aligned}
\end{equation*}
Since $\| \partial \ell / \partial f\|_{2} = \sqrt{ (\partial \ell / \partial f_{1})^{2}+\cdots+(\partial \ell / \partial f_{k})^{2} }$, then for each class $c$, when $y=c$, it follows that 
\begin{equation*}
\begin{aligned}
\sum_{j}(\partial \ell / \partial f_{j})^{2} &= \frac{(\sum_{j\neq c}\exp(f_{j}))^{2}}{(\sum_j\exp(f_{j}))^{2}} + \frac{\sum_{l \neq c}(\exp(f_{l}))^{2}}{(\sum_j\exp(f_{j}))^{2}} \\
&\leq \frac{(\sum_{j\neq c}\exp(f_{j}))^{2}}{(\sum_j\exp(f_{j}))^{2}} + \frac{(\sum_{l \neq c} \exp(f_{l}))^{2}}{(\sum_j\exp(f_{j}))^{2}} \\
&= 2 \frac{(\sum_{j\neq c}\exp(f_{j}))^{2}}{(\sum_j\exp(f_{j}))^{2}} \\
&\leq 2.
\end{aligned}
\end{equation*}
Finally, one sees that $L_{\ell_{2}} = \sup \{ \| \partial \ell / \partial f \|_{2} \} \leq \sqrt{2}$. Similarly, one sees that $L_{\ell_{\infty}} = \sup \{ \| \partial \ell / \partial f \|_{\infty} \} \leq 1$.
\end{proof}

\begin{proof}[Proof of Lemma \ref{lemma-decomposition}]

\begin{equation*}
\begin{aligned}
\mathcal{R}_{S}(\hat{\ell}_{2\mathcal{F}_{r_{2}}}) &= \E_{\xi}\left[\sup\limits_{f \in \mathcal{F}_{r_{2}}} \frac{1}{n}\sum\limits_{i=1}^n \xi_{i} \hat{\ell}_{2}(f(x_{i}),y_{i}) \right] \\
&= \E_{\xi}\left[\sup\limits_{f \in \mathcal{F}_{r_{2}}} \frac{1}{n}\sum\limits_{i=1}^n \xi_{i} (\ell(f(x_{i}),y_{i}) + \frac12\lambda\epsilon L_{\ell_{2}}^{2} \| \nabla_{x}f(x) \|_{F}^{2} ) \right] \\
&\leq \E_{\xi}\left[\sup\limits_{f \in \mathcal{F}_{r_{2}}} \frac{1}{n}\sum\limits_{i=1}^n \xi_{i} \ell(f(x_{i}),y_{i}) \right] + \E_{\xi}\left[\sup\limits_{f \in \mathcal{F}_{r_{2}}} \frac{1}{n}\sum\limits_{i=1}^n \xi_{i}  \frac12\lambda\epsilon L_{\ell_{2}}^{2} \| \nabla_{x}f(x) \|_{F}^{2} \right] \\
&=\mathcal{R}_{S}(\ell_{\mathcal{F}_{r_{2}}})+\frac12\lambda\epsilon L_{\ell_{2}}^{2} \mathcal{R}_{S}(\mathcal{J}_{2\mathcal{F}_{r_{2}}}).
\end{aligned}
\end{equation*}

\begin{equation*}
\begin{aligned}
\mathcal{R}_{S}(\hat{\ell}_{\infty\mathcal{F}_{r_{1}}}) &= \E_{\xi}\left[\sup\limits_{f \in \mathcal{F}_{r_{1}}} \frac{1}{n}\sum\limits_{i=1}^n \xi_{i} \hat{\ell}_{\infty}(f(x_{i}),y_{i}) \right] \\
&= \E_{\xi}\left[\sup\limits_{f \in \mathcal{F}_{r_{1}}} \frac{1}{n}\sum\limits_{i=1}^n \xi_{i} (\ell(f(x_{i}),y_{i}) + \lambda\epsilon L_{\ell_{\infty}} \| \nabla_{x}f(x) \|_{1,1} ) \right] \\
&\leq \E_{\xi}\left[\sup\limits_{f \in \mathcal{F}_{r_{1}}} \frac{1}{n}\sum\limits_{i=1}^n \xi_{i} \ell(f(x_{i}),y_{i}) \right] + \E_{\xi}\left[\sup\limits_{f \in \mathcal{F}_{r_{1}}} \frac{1}{n}\sum\limits_{i=1}^n \xi_{i}  \lambda\epsilon L_{\ell_{\infty}} \| \nabla_{x}f(x) \|_{1,1} \right] \\
&=\mathcal{R}_{S}(\ell_{\mathcal{F}_{r_{1}}})+\lambda\epsilon L_{\ell_{\infty}} \mathcal{R}_{S}(\mathcal{J}_{1 \mathcal{F}_{r_{1}}}).
\end{aligned}
\end{equation*}
\end{proof}

\begin{proof}[Proof of Lemma \ref{lipschitz-parameter}]
Denote
\begin{equation*}
\nabla_{\theta} \ell(f_{\Theta}(x),y) = \left(  \frac{\partial \ell}{\partial \theta_{1}}\Big|_{\text{vec}}^{\top}, \cdots,  \frac{\partial \ell}{\partial \theta_{L}}\Big|_{\text{vec}}^{\top}       \right)^{\top},
\end{equation*}
where for $i=1,2,\cdots,L$, $\frac{\partial \ell}{\partial \theta_{i}}\Big|_{\text{vec}}$ is the vectorized form of the matrix $\frac{\partial \ell}{\partial \theta_{i}}$.  
Since the Lipschitz constant equals to $ \sup\{ \| \nabla_{\theta} \ell(f_{\Theta}(x),y) \|_{2}: \Theta \in \varTheta \}$, it suffice to analyze the $\ell_{2}$-norm of $\nabla_{\theta} \ell(f_{\Theta}(x),y)$ defined as 
\begin{equation*}
\| \nabla_{\theta} \ell(f_{\Theta}(x),y) \|_{2} = \sqrt{ \sum_{l=1}^{L} \| \frac{\partial \ell}{\partial \theta_{l}} \|_{F}^{2}  }.
\end{equation*}
For a deep neural network $f_\Theta(x)=\theta_L\sigma(\theta_{L-1}\cdots\sigma(\theta_1 x)\cdots)$, denote the activations of the $l$-th hidden layer as $h_{l}=\sigma(\theta_{l}\cdots\sigma(\theta_1 x)\cdots)$ with $h_{0}=x$. Then the derivatives of the $l$-th hidden layer is $h'_{l}=\sigma'(\theta_{l}\cdots\sigma(\theta_1 x)\cdots)$, where $\sigma'(\cdot)$ is the derivative of the activation function. Denote $\text{diag}(h'_{l})$ as the diagonal matrix with the diagonal elements being $h'_{l}$, according to the chain rule, it follows that 
\begin{equation*}
\frac{\partial \ell}{\partial \theta_{l}} =\left( \text{diag}(h'_{l}) \theta_{l+1}^{\top} \cdots \theta_{L-1}^{\top} \text{diag}(h'_{L-1}) \theta_{L}^{\top} \frac{\partial \ell}{\partial f} \right)  h_{l-1}^{\top}.
\end{equation*}
For the Relu activation function $\sigma(x)=\max\{0,x\}$, the derivative $\sigma'(\cdot)$ is the Step function, thus $\sigma'(\cdot) \in [0,1]$ and $h'_{l} \in [0,1]$. The Frobenius-norm of gradients is bounded layer by layer as:
\begin{equation}\label{lip-model1}
\begin{aligned}
\| \frac{\partial f}{\partial \theta_{l}} \|_{F} &= \sqrt{\text{trace}(\frac{\partial f}{\partial \theta_{l}}  \frac{\partial f}{\partial \theta_{l}}^{\top})} \\
&= \| \text{diag}(h'_{l}) \theta_{l+1}^{\top} \cdots \theta_{L-1}^{\top} \text{diag}(h'_{L-1}) \theta_{L}^{\top} \frac{\partial \ell}{\partial f} \|_{2}  \| h_{l-1} \|_{2} \\
&\leq \| \theta_{l+1}^{\top} \cdots \theta_{L-1}^{\top} \text{diag}(h'_{L-1}) \theta_{L}^{\top} \frac{\partial \ell}{\partial f} \|_{2}  \| h_{l-1} \|_{2}  \quad (\because h'_{l} \in [0,1])\\
&\leq \| \theta_{l+1} \|_{F}  \| \text{diag}(h'_{l+1}) \cdots \theta_{L-1}^{\top} \text{diag}(h'_{L-1}) \theta_{L}^{\top} \frac{\partial \ell}{\partial f} \|_{2}  \| h_{l-1} \|_{2} \\
&\leq \| h_{l-1} \|_{2} \| \frac{\partial \ell}{\partial f} \|_{2} \prod_{k=l+1}^{L} \| \theta_{k} \|_{F}  \\
&\leq L_{\ell_{2}} \| h_{l-1} \|_{2} \prod_{k=l+1}^{L} \| \theta_{k} \|_{F}.
\end{aligned}
\end{equation}
The term $\| h_{l-1} \|_{2}$ is also bounded layer by layer as:
\begin{equation}\label{lip-model2}
\begin{aligned}
\| h_{l-1} \|_{2} &= \| \sigma(\theta_{l-1}\cdots\sigma(\theta_1 x)\cdots) \|_{2} \\
&\leq \| \theta_{l-1} \sigma(\theta_{l-2} \cdots\sigma(\theta_1 x)\cdots) \|_{2} \\
&\leq \| \theta_{l-1} \|_{F}  \| \sigma(\theta_{l-2} \cdots\sigma(\theta_1 x)\cdots) \|_{2} \\
&\leq \| x \|_{2} \prod_{k=1}^{l-1} \| \theta_{k} \|_{F},
\end{aligned}
\end{equation}
where the first inequality is due to the 1-Lipschitz property of $\sigma(\cdot)$ and fact that $\sigma(0)=0$, and the last inequality is by iteration. Combining \eqref{lip-model1} and \eqref{lip-model2} yields that
\begin{equation*}
\| \frac{\partial \ell}{\partial \theta_{l}} \|_{F} \leq L_{\ell_{2}} \| x \|_{2} \prod_{k=1,k \neq l}^{L} \| \theta_{k} \|_{F}.
\end{equation*}
Combining the Frobenius-norm of gradients of all layers, we obtain that 
\begin{equation*}
\begin{aligned}
\sqrt{ \sum_{l=1}^{L} \| \frac{\partial \ell}{\partial \theta_{l}} \|_{F}^{2}  } &\leq L_{\ell_{2}} \| x \|_{2} \sqrt{ \sum_{l=1}^{L} \prod_{k=1,k \neq l}^{L} \| \theta_{k} \|_{F}^{2}} \\
&\leq L_{\ell_{2}} \| x \|_{2} \sqrt{ L \max_{l \in \{1,\dots,L \}} \prod_{k=1,k \neq l}^{L} \| \theta_{k} \|_{F}^{2}}   \\
&\leq \sqrt{L} L_{\ell_{2}} \| x \|_{2} \sqrt{ \max_{l \in \{1,\dots,L \}} \left( \frac{1}{L-1} \sum_{k=1,k \neq l}^{L} \| \theta_{k} \|_{F}^{2} \right)^{L-1}  } \\
&\leq \sqrt{L} L_{\ell_{2}} \| x \|_{2} \sqrt{ \left( \frac{1}{L-1} \sum_{k=1}^{L} \| \theta_{k} \|_{F}^{2} \right)^{L-1}  } \\
&\leq \sqrt{L} L_{\ell_{2}} \| x \|_{2} \sqrt{ \left( \frac{1}{L-1} R_{\Theta}^{2} \right)^{L-1}  } \\
&= \sqrt{L} L_{\ell_{2}} \| x \|_{2}  \left( \frac{R_{\Theta}}{\sqrt{L-1}} \right)^{L-1}.  
%&=  \sqrt{L} L_{\ell_{2}} \| x \|_{2} \max_{l \in \{1,\dots,L \}} \prod_{k=1,k \neq l}^{L} \| \theta_{k} \|_{F} .\\
\end{aligned}
\end{equation*}
Furthermore, it follows that
\begin{equation*}
\begin{aligned}
& \| \ell(f_{\Theta}(x),y) - \ell(f_{\Theta'}(x),y) \|_{L_{1}(P_{n})} = \frac{1}{n} \sum_{i=1}^{n} | \ell(f_{\Theta}(x_{i}),y_{i}) - \ell(f_{\Theta'}(x_{i}),y_{i}) | \\
&\leq \frac{1}{n} \sum_{i=1}^{n} \sqrt{L} L_{\ell_{2}} \| x_{i} \|_{2} \left( \frac{R_{\Theta}}{\sqrt{L-1}} \right)^{L-1} \norm{\Theta-\Theta'}_{F}  \\
&= \sqrt{L} L_{\ell_{2}}  \frac{1}{n} \sum_{i=1}^{n} \| x_{i} \|_{2} \left( \frac{R_{\Theta}}{\sqrt{L-1}} \right)^{L-1}   \norm{\Theta-\Theta'}_{F}.
\end{aligned}
\end{equation*}
The proof is complete.
\end{proof}

\begin{proof}[Proof of Lemma \ref{lipschitz-jacobian}]
Recall that the derivatives of the $l$-th hidden layer is $h'_{l}=\sigma'(\theta_{l}\cdots\sigma(\theta_1 x)\cdots)$, then it follows from the chain rule that the gradients with respect to inputs are equal to
\begin{equation*}
\nabla_{x}f(x) = \theta_{1}^{\top} \text{diag}(h'_{1}) \theta_{2}^{\top} \cdots \theta_{L-1}^{\top} \text{diag}(h'_{L-1}) \theta_{L}^{\top}.
\end{equation*}
The gradients of the $j$-th output of $f(x)$ is
\begin{equation*}
\nabla_{x}f_{j}(x) = \theta_{1}^{\top} \text{diag}(h'_{1}) \theta_{2}^{\top} \cdots \theta_{L-1}^{\top} \text{diag}(h'_{L-1}) \theta_{Lj\cdot}^{\top},
\end{equation*}
where we use $\theta_{Lj\cdot}$ to denote the $j$-th row of $\theta_{L}$. Then it holds that
\begin{equation*}
\| \nabla_{x}f(x) \|_{F}^{2} = \sum_{j=1}^{k} \| \nabla_{x}f_{j}(x) \|_{2}^{2}\quad \text{and}\quad  \| \nabla_{x}f(x) \|_{1,1} = \sum_{j=1}^{k} \| \nabla_{x}f_{j}(x) \|_{1}.
\end{equation*}
We firstly calculate the derivatives of $\| \nabla_{x}f(x) \|_{F}^{2}$. Since $\nabla_{x}f_{j}(x)$ is a column vector, it follows that
\begin{equation*}
\begin{aligned}
&\| \nabla_{x}f_{j}(x) \|_{2}^{2} =  \nabla_{x}f_{j}^{\top}(x)  \nabla_{x}f_{j}(x) \\  &=\theta_{Lj\cdot} \text{diag}(h'_{L-1}) \theta_{L-1} \cdots \theta_{2} \text{diag}(h'_{1}) \theta_{1}   \theta_{1}^{\top} \text{diag}(h'_{1}) \theta_{2}^{\top} \cdots \theta_{L-1}^{\top} \text{diag}(h'_{L-1}) \theta_{Lj\cdot}^{\top}.
\end{aligned}
\end{equation*}
Next, we calculate the derivatives of $\| \nabla_{x}f(x) \|_{F}^{2}$ with respect to parameters to characterize the Lipschitz property with respect to parameters.
For parameter $\theta_{L}$, since $\theta_{Lj\cdot}$ only exists in $\nabla_{x}f_{j}(x)$,  one has that
\begin{equation*}
\begin{aligned}
&\frac{\partial \| \nabla_{x}f(x) \|_{F}^{2}}{\partial \theta_{Lj\cdot}} = \frac{\partial \| \nabla_{x}f_{j}(x) \|_{2}^{2}} {\partial \theta_{Lj\cdot}} \\  
&=2 \theta_{Lj\cdot} \text{diag}(h'_{L-1}) \theta_{L-1} \cdots \theta_{2} \text{diag}(h'_{1}) \theta_{1}   \theta_{1}^{\top} \text{diag}(h'_{1}) \theta_{2}^{\top} \cdots \theta_{L-1}^{\top} \text{diag}(h'_{L-1}) .
\end{aligned}
\end{equation*}
Then, the norm of derivatives with respect to parameter $\theta_{Lj\cdot}$ is bounded by
\begin{equation*}
\begin{aligned}
\left \| \frac{\partial \| \nabla_{x}f(x) \|_{F}^{2}}{\partial \theta_{Lj\cdot}} \right \|_{2} &=2 \| \theta_{Lj\cdot} \text{diag}(h'_{L-1}) \theta_{L-1} \cdots \theta_{1}   \theta_{1}^{\top}  \cdots \theta_{L-1}^{\top} \text{diag}(h'_{L-1}) \|_{2} \\
&\leq 2 \| \theta_{Lj\cdot} \text{diag}(h'_{L-1}) \theta_{L-1} \cdots \theta_{1}   \theta_{1}^{\top}  \cdots \theta_{L-1}^{\top} \|_{2} \\
&\leq 2 \| \theta_{Lj\cdot} \text{diag}(h'_{L-1}) \theta_{L-1} \cdots \theta_{1}   \theta_{1}^{\top}  \cdots  \text{diag}(h'_{L-2}) \|_{2} \|\theta_{L-1}\|_{F} \\
&\leq 2 \|\theta_{Lj\cdot}\|_{2} \prod_{l=1}^{L-1} \| \theta_{l} \|_{F}^{2},
\end{aligned}
\end{equation*}
and thus it follows directly that
\begin{equation*}
\begin{aligned}
\left \| \frac{\partial \| \nabla_{x}f(x) \|_{F}^{2}}{\partial \theta_{L}} \right \|_{F}^{2} &= \sum_{j=1}^{k}\left \| \frac{\partial \| \nabla_{x}f(x) \|_{F}^{2}}{\partial \theta_{Lj\cdot}} \right \|_{2}^{2} \\
&\leq \sum_{j=1}^{k} 4 \|\theta_{Lj\cdot}\|_{2}^{2} \prod_{l=1}^{L-1} \| \theta_{l} \|_{F}^{4} = 4\|\theta_{L}\|_{F}^{2} \prod_{l=1}^{L-1} \| \theta_{l} \|_{F}^{4}.
\end{aligned}
\end{equation*}
It is worth noting that other parameters $\theta_{L-1},\cdots,\theta_{1}$ also exist in $\text{diag}(h'_{L-1}),\cdots,\text{diag}(h'_{1})$. Since we assume that there are no pre-activations $h_{l_{pre}} = (\theta_{l}\cdots\sigma(\theta_1 x)\cdots)$ which are exactly 0, the second derivatives $h''_{l}=\sigma''(\theta_{l}\cdots\sigma(\theta_1 x)\cdots)$ is 0 for Relu activation function, and thus the derivatives $\partial \text{diag}(h'_{l}) / \partial \theta_{j} = 0$. Then the derivatives of other parameters satisfy 
\begin{equation*}
\begin{aligned}
&\frac{\partial \| \nabla_{x}f(x) \|_{F}^{2}}{\partial \theta_{l}} = \sum_{j=1}^{k}\frac{\partial \| \nabla_{x}f_{j}(x) \|_{2}^{2}} {\partial \theta_{l}} \\  
%&=\theta_{Lj\cdot} \text{diag}(h'_{L-1}) \theta_{L-1} \cdots \theta_{2} \text{diag}(h'_{1}) \theta_{1}   \theta_{1}^{\top} \text{diag}(h'_{1}) \theta_{2}^{\top} \cdots \theta_{L-1}^{\top} \text{diag}(h'_{L-1}) \theta_{Lj\cdot}^{\top}  \\
&=2\sum_{j=1}^{k}  \text{diag}(h'_{l}) \theta_{l+1}^{\top} \cdots  \theta_{Lj\cdot}^{\top} \theta_{Lj\cdot} \text{diag}(h'_{L-1}) \theta_{L-1} \cdots  \theta_{1}  \theta_{1}^{\top} \cdots \theta_{l-1}^{\top} \text{diag}(h'_{l-1}) \\
&=2\text{diag}(h'_{l}) \theta_{l+1}^{\top} \cdots  \theta_{L}^{\top} \theta_{L} \text{diag}(h'_{L-1}) \theta_{L-1} \cdots  \theta_{1}  \theta_{1}^{\top} \cdots \theta_{l-1}^{\top} \text{diag}(h'_{l-1}). \\
\end{aligned}
\end{equation*}
Then, the norm of derivatives with respect to parameter $\theta_{l}$ satisfy
\begin{equation*}
\begin{aligned}
\left \| \frac{\partial \| \nabla_{x}f(x) \|_{F}^{2}}{\partial \theta_{l}} \right \|_{F} &=2 \| \text{diag}(h'_{l}) \theta_{l+1}^{\top} \cdots  \theta_{L}^{\top} \theta_{L} \cdots  \theta_{1}  \theta_{1}^{\top} \cdots \theta_{l-1}^{\top} \text{diag}(h'_{l-1}) \|_{F} \\
&\leq 2 \| \text{diag}(h'_{l}) \theta_{l+1}^{\top} \cdots  \theta_{L}^{\top} \theta_{L} \cdots  \theta_{1}  \theta_{1}^{\top} \cdots \theta_{l-1}^{\top} \|_{F} \| \text{diag}(h'_{l-1}) \|_{\text{op}} \\
&\leq 2 \| \text{diag}(h'_{l}) \theta_{l+1}^{\top} \cdots  \theta_{L}^{\top} \theta_{L} \cdots  \theta_{1}  \theta_{1}^{\top} \cdots \theta_{l-1}^{\top} \|_{F} \\
&\leq 2 \| \text{diag}(h'_{l}) \theta_{l+1}^{\top} \cdots  \theta_{L}^{\top} \theta_{L} \cdots  \theta_{1}  \theta_{1}^{\top} \cdots \text{diag}(h'_{l-2}) \|_{F} \|\theta_{l-1}\|_{\text{op}} \\
&\leq 2 \| \text{diag}(h'_{l}) \theta_{l+1}^{\top} \cdots  \theta_{L}^{\top} \theta_{L} \cdots  \theta_{1}  \theta_{1}^{\top} \cdots \text{diag}(h'_{l-2}) \|_{F} \|\theta_{l-1}\|_{F} \\
&\leq 2 \|\theta_{l}\|_{F} \prod_{j=1,j\neq l}^{L} \| \theta_{j} \|_{F}^{2}.
\end{aligned}
\end{equation*}
Thus, for all $l=1,\cdots,L$, it follows that
\begin{equation*}
\left \| \frac{\partial \| \nabla_{x}f(x) \|_{F}^{2}}{\partial \theta_{l}} \right \|_{F}^{2} \leq 4 \|\theta_{l}\|_{F}^{2} \prod_{j=1,j\neq l}^{L} \| \theta_{j} \|_{F}^{4}. 
\end{equation*}
Similarly, denote the vectorized form of derivatives as 
\begin{equation*}
\nabla_{\theta}\| \nabla_{x}f(x) \|_{F}^{2} = \left(  \frac{\partial \| \nabla_{x}f(x) \|_{F}^{2}}{\partial \theta_{1}}\Big|_{\text{vec}}^{\top}, \cdots,  \frac{\partial \| \nabla_{x}f(x) \|_{F}^{2}}{\partial \theta_{L}}\Big|_{\text{vec}}^{\top}       \right)^{\top}.
\end{equation*}
Then the Lipschitz property with respect to parameters can be bounded by
\begin{equation*}
\begin{aligned}
\| \nabla_{\theta}\| \nabla_{x}f(x) \|_{F}^{2} \|_{2} &= \sqrt{ \sum_{l=1}^{L} \left \| \frac{\partial \| \nabla_{x}f(x) \|_{F}^{2}}{\partial \theta_{l}} \right \|_{F}^{2}  } \\
&\leq \sqrt{ \sum_{l=1}^{L} 4 \|\theta_{l}\|_{F}^{2} \prod_{j=1,j\neq l}^{L} \| \theta_{j} \|_{F}^{4}  } \\
&\leq \sqrt{ L \max_{l \in \{1,\dots,L \}} 4 \|\theta_{l}\|_{F}^{2} \prod_{j=1,j\neq l}^{L} \| \theta_{j} \|_{F}^{4}  } \\
&\leq \sqrt{ 4L \max_{l \in \{1,\dots,L \}} \left(\frac{1}{2L-1} (\|\theta_{l}\|_{F}^{2}+ \sum_{j=1,j\neq l}^{L} 2\| \theta_{j} \|_{F}^{2}) \right)^{2L-1} } \\
&\leq \sqrt{ 4L  \left(\frac{1}{2L-1} (2\|\theta_{l}\|_{F}^{2}+ \sum_{j=1,j\neq l}^{L} 2\| \theta_{j} \|_{F}^{2}) \right)^{2L-1} } \\
&\leq2\sqrt{L} \left( \sqrt{\frac{2}{2L-1} R_{\Theta}^{2}} \right)^{2L-1} \\
&\leq 2\sqrt{L} \left(\frac{R_{\Theta}}{\sqrt{L-1}}\right)^{2L-1}.
\end{aligned}
\end{equation*}
Finally, it follows that
\begin{equation*}
\begin{aligned}
\| \|\nabla_{x}f_{\Theta}(x)\|^{2}_{F} - \|\nabla_{x}f_{\Theta'}(x)\|^{2}_{F} \|_{L_{1}(P_{n})} &= \frac1n \sum_{i=1}^{n} \left| \|\nabla_{x}f_{\Theta}(x_{i})\|^{2}_{F} - \|\nabla_{x}f_{\Theta'}(x_{i})\|^{2}_{F} \right| \\
&\leq 2\sqrt{L} \left(\frac{R_{\Theta}}{\sqrt{L-1}}\right)^{2L-1} \norm{\Theta-\Theta'}_{F} .
\end{aligned}
\end{equation*}
Next we calculate the derivatives of $\| \nabla_{x}f(x) \|_{1,1}$. Recall that the gradients of the $j$-th output of $f(x)$ is
\begin{equation*}
\nabla_{x}f_{j}(x) = \theta_{1}^{\top} \text{diag}(h'_{1}) \theta_{2}^{\top} \cdots \theta_{L-1}^{\top} \text{diag}(h'_{L-1}) \theta_{Lj\cdot}^{\top}.
\end{equation*}
Then it holds that
\begin{equation*}
 \| \nabla_{x}f(x) \|_{1,1} = \sum_{j=1}^{k} \| \nabla_{x}f_{j}(x) \|_{1} = \sum_{j=1}^{k} \| \theta_{1}^{\top} \text{diag}(h'_{1}) \theta_{2}^{\top} \cdots \theta_{L-1}^{\top} \text{diag}(h'_{L-1}) \theta_{Lj\cdot}^{\top} \|_{1}.
\end{equation*}
For parameter $\theta_{L}$, since $\theta_{Lj\cdot}$ only exists in $\nabla_{x}f_{j}(x)$, one has that
\begin{equation*}
\begin{aligned}
&\frac{\partial \| \nabla_{x}f(x) \|_{1,1}}{\partial \theta_{Lj\cdot}} = \frac{\partial \| \nabla_{x}f_{j}(x) \|_{1}} {\partial \theta_{Lj\cdot}} \\  
&= \text{diag}(h'_{L-1}) \theta_{L-1} \cdots \theta_{2} \text{diag}(h'_{1}) \theta_{1}  \text{sign}(\nabla_{x}f_{j}(x))  .
\end{aligned}
\end{equation*}
Then, the norm of derivatives with respect to parameter $\theta_{Lj\cdot}$ is bounded by
\begin{equation*}
\begin{aligned}
\left \| \frac{\partial \| \nabla_{x}f(x) \|_{1,1}}{\partial \theta_{Lj\cdot}} \right \|_{2} &= \| \text{diag}(h'_{L-1}) \theta_{L-1} \cdots \theta_{2} \text{diag}(h'_{1}) \theta_{1}  \text{sign}(\nabla_{x}f_{j}(x)) \|_{2} \\
&\leq \| \theta_{L-1} \cdots \theta_{2} \text{diag}(h'_{1}) \theta_{1}  \text{sign}(\nabla_{x}f_{j}(x)) \|_{2} \\
&\leq \|\theta_{L-1}\|_{F} \| \text{diag}(h'_{L-2}) \cdots \theta_{2} \text{diag}(h'_{1}) \theta_{1}  \text{sign}(\nabla_{x}f_{j}(x)) \|_{2} \\
&\leq \| \text{sign}(\nabla_{x}f_{j}(x)) \|_{2} \prod_{l=1}^{L-1} \| \theta_{l} \|_{F} \\
&=\sqrt{d}  \prod_{l=1}^{L-1} \| \theta_{l} \|_{F}.
\end{aligned}
\end{equation*}
and thus it follows directly that
\begin{equation*}
\begin{aligned}
\left \| \frac{\partial \| \nabla_{x}f(x) \|_{1,1}}{\partial \theta_{L}} \right \|_{F}^{2} &= \sum_{j=1}^{k}\left \| \frac{\partial \| \nabla_{x}f(x) \|_{1,1}}{\partial \theta_{Lj\cdot}} \right \|_{2}^{2} \leq kd \prod_{l=1}^{L-1} \| \theta_{l} \|_{F}^{2} .
\end{aligned}
\end{equation*}
For other parameters $\theta_{L-1},\cdots,\theta_{1}$, it holds that
\begin{equation*}
\begin{aligned}
&\frac{\partial \| \nabla_{x}f(x) \|_{1,1}}{\partial \theta_{l}} = \sum_{j=1}^{k}\frac{\partial \| \nabla_{x}f_{j}(x) \|_{1}} {\partial \theta_{l}} \\  
&=\sum_{j=1}^{k} \text{diag}(h'_{l}) \theta_{l+1}^{\top}  \cdots \text{diag}(h'_{L-1}) \theta_{Lj\cdot}^{\top} \text{sign}(\nabla_{x}f_{j}(x))^{\top} \theta_{1}^{\top}  \cdots \theta_{l-1}^{\top}  \text{diag}(h'_{l-1}) \\
&= \text{diag}(h'_{l}) \theta_{l+1}^{\top}  \cdots \text{diag}(h'_{L-1}) \theta_{L}^{\top} \text{sign}(\nabla_{x}f(x))^{\top} \theta_{1}^{\top}   \cdots \theta_{l-1}^{\top}  \text{diag}(h'_{l-1}),
\end{aligned}
\end{equation*}
and thus
\begin{equation*}
\begin{aligned}
\left \| \frac{\partial \| \nabla_{x}f(x) \|_{1,1}}{\partial \theta_{l}} \right \|_{F} &= \| \text{diag}(h'_{l}) \theta_{l+1}^{\top}  \cdots \theta_{L}^{\top} \text{sign}(\nabla_{x}f(x))^{\top} \theta_{1}^{\top}  \cdots \theta_{l-1}^{\top}  \text{diag}(h'_{l-1}) \|_{F} \\
&\leq \| \text{diag}(h'_{l}) \theta_{l+1}^{\top}  \cdots \theta_{L}^{\top} \text{sign}(\nabla_{x}f(x))^{\top} \theta_{1}^{\top}  \cdots \theta_{l-1}^{\top} \|_{F} \|  \text{diag}(h'_{l-1}) \|_{\text{op}} \\
&\leq \| \text{diag}(h'_{l}) \theta_{l+1}^{\top}  \cdots \theta_{L}^{\top} \text{sign}(\nabla_{x}f(x))^{\top} \theta_{1}^{\top}  \cdots \theta_{l-1}^{\top} \|_{F} \\
&\leq \| \text{diag}(h'_{l}) \theta_{l+1}^{\top}  \cdots \theta_{L}^{\top} \text{sign}(\nabla_{x}f(x))^{\top} \theta_{1}^{\top}  \cdots \text{diag}(h'_{l-2}) \|_{F} \|\theta_{l-1}^{\top} \|_{\text{op}} \\
&\leq \| \text{diag}(h'_{l}) \theta_{l+1}^{\top}  \cdots \theta_{L}^{\top} \text{sign}(\nabla_{x}f(x))^{\top} \theta_{1}^{\top}  \cdots \text{diag}(h'_{l-2}) \|_{F} \|\theta_{l-1}^{\top} \|_{F} \\
&\leq \| \text{sign}(\nabla_{x}f(x)) \|_{F} \prod_{j=1,j\neq l}^{L} \|\theta_{j}\|_{F} \\
&\leq \sqrt{kd} \prod_{j=1,j\neq l}^{L} \|\theta_{j}\|_{F}.
\end{aligned}
\end{equation*}
Hence one has that
\begin{equation*}
\left \| \frac{\partial \| \nabla_{x}f(x) \|_{1,1}}{\partial \theta_{l}} \right \|_{F}^{2} \leq kd \prod_{j=1,j\neq l}^{L} \|\theta_{j}\|_{F}^{2}.
\end{equation*}
Similarly, denote the vectorized form of derivatives as 
\begin{equation*}
\nabla_{\theta}\| \nabla_{x}f(x) \|_{1,1} = \left(  \frac{\partial \| \nabla_{x}f(x) \|_{1,1}}{\partial \theta_{1}}\Big|_{\text{vec}}^{\top}, \cdots,  \frac{\partial \| \nabla_{x}f(x) \|_{1,1}}{\partial \theta_{L}}\Big|_{\text{vec}}^{\top}       \right)^{\top}.
\end{equation*}
Then the Lipschitzness with respect to parameters can be bounded by
\begin{equation*}
\begin{aligned}
\| \nabla_{\theta}\| \nabla_{x}f(x) \|_{1,1} \|_{2} &= \sqrt{ \sum_{l=1}^{L} \left \| \frac{\partial \| \nabla_{x}f(x) \|_{1,1}}{\partial \theta_{l}} \right \|_{F}^{2}  } \\
&\leq \sqrt{ \sum_{l=1}^{L} kd \prod_{j=1,j\neq l}^{L} \|\theta_{j}\|_{F}^{2}  } \\
&\leq \sqrt{ Lkd \max_{l \in \{1,\dots,L \}}  \prod_{j=1,j\neq l}^{L} \| \theta_{j} \|_{F}^{2}  } \\
&\leq \sqrt{ Lkd \max_{l \in \{1,\dots,L \}} \left(\frac{1}{L-1} \sum_{j=1,j\neq l}^{L} \| \theta_{j} \|_{F}^{2} \right)^{L-1} } \\
&\leq \sqrt{ Lkd  \left(\frac{1}{L-1} \sum_{j=1}^{L} \| \theta_{j} \|_{F}^{2} \right)^{L-1} } \\
&\leq \sqrt{Lkd} \left(\frac{R_{\Theta}}{\sqrt{L-1}}\right)^{L-1}.
\end{aligned}
\end{equation*}
Finally, it follows that
\begin{equation*}
\begin{aligned}
\| \|\nabla_{x}f_{\Theta}(x)\|_{1,1} - \|\nabla_{x}f_{\Theta'}(x)\|_{1,1} \|_{L_{1}(P_{n})} &= \frac1n \sum_{i=1}^{n} \left| \|\nabla_{x}f_{\Theta}(x_{i})\|_{1,1} - \|\nabla_{x}f_{\Theta'}(x_{i})\|_{1,1} \right| \\
&\leq \sqrt{Lkd} \left(\frac{R_{\Theta}}{\sqrt{L-1}}\right)^{L-1} \norm{\Theta-\Theta'}_{F} .
\end{aligned}
\end{equation*}
The proof is complete.
\end{proof}

\begin{proof}[Proof of Lemma \ref{lemma-derivative}]
For a deep neural network $f(x)=\theta_L\sigma(\theta_{L-1}\cdots\sigma(\theta_1 x)\cdots)$, denote the activations of the $l$-th hidden layer as $h_{l}=\sigma(\theta_{l}\cdots\sigma(\theta_1 x)\cdots)$ with $h_{0}=x$. Then the derivatives of the $l$-th hidden layer is $h'_{l}=\sigma'(\theta_{l}\cdots\sigma(\theta_1 x)\cdots)$, where $\sigma'(\cdot)$ is the derivative of the activation function. When $\sigma(x)=\text{Relu}(x)$, the derivative is the step function $\sigma'(x)=\text{Step}(x)$, where $\text{Step}(x)=1$ when $x>0$ and $\text{Step}(x)=0$ when $x\leq0$. Denote $\text{diag}(h'_{l})$ as the diagonal matrix with the diagonal elements being $h'_{l}$,  then it follows that
\begin{equation*}
\begin{aligned}
%\nabla_{x}f(x) &= \theta_{1}^{\top}  h'_{1} \odot \theta_{2}^{\top} \cdots \theta_{L-1}^{\top} h'_{L-1} \odot \theta_{L}^{\top}, \\
\nabla_{x}f(x) &= \theta_{1}^{\top} \text{diag}(h'_{1}) \theta_{2}^{\top} \cdots \theta_{L-1}^{\top} \text{diag}(h'_{L-1}) \theta_{L}^{\top}, \\
\nabla_{x}f(x)^{ \top} x &=\theta_{L} \text{diag}(h'_{L-1}) \theta_{L-1} \cdots \theta_{2}  \text{diag}(h'_{1}) \theta_{1} x.
\end{aligned}
\end{equation*}
Note that $\text{diag}(h'_{1}) \theta_{1} x = \text{diag}(\sigma'(\theta_{1} x)) \theta_{1} x=\sigma(\theta_{1} x)$. Applying this rule recursively, we obtian that $f(x) = \nabla_{x}f(x)^{ \top} x$.
\end{proof}

\begin{proof}[Proof of Lemma \ref{lemma-variance}]
It follows from Lemma \ref{lemma-derivative} that
\begin{equation*}
\begin{aligned}
P_{n}\|f\|_{2} &=\frac1n \sum_{i=1}^{n} \| \nabla_{x}f(x_{i})^{ \top} x_{i} \|_{2} \\
&\leq \frac1n \sum_{i=1}^{n} \| \nabla_{x}f(x_{i})\|_{\text{op}} \| x_{i} \|_{2}  \\
&\leq \frac1n \sum_{i=1}^{n} \| \nabla_{x}f(x_{i})\|_{F} \| x_{i} \|_{2}  \\
&\leq \sqrt{\frac1n \sum_{i=1}^{n} \| \nabla_{x}f(x_{i})\|_{F}^{2} } \sqrt{\frac1n \sum_{i=1}^{n} \| x_{i} \|_{2}^{2} } \\
&\leq \sqrt{r_{2}} \sqrt{\frac1n \sum_{i=1}^{n} \| x_{i} \|_{2}^{2} }.
\end{aligned}
\end{equation*}

\begin{equation*}
\begin{aligned}
P_{n}\|f\|_{1}=\frac1n \sum_{i=1}^{n} \sum_{j=1}^{k} | f_{j}(x_{i}) | &=\frac1n \sum_{i=1}^{n} \sum_{j=1}^{k} |\nabla_{x}f_{j}(x_{i})^{ \top} x_{i}| \\
&\leq \frac1n \sum_{i=1}^{n} \sum_{j=1}^{k}  \| \nabla_{x}f_{j}(x_{i}) \|_{1} \|x_{i}\|_{\infty} \\
&= \frac1n \sum_{i=1}^{n} \| \nabla_{x}f(x_{i}) \|_{1,1} \|x_{i}\|_{\infty} \\
&\leq \frac1n \sum_{i=1}^{n} \| \nabla_{x}f(x_{i}) \|_{1,1} R_{x} \\
&\leq r_{1}R_{x}.
\end{aligned}
\end{equation*}
\end{proof}

\begin{proof}[Proof of Lemma \ref{lemma-variance-loss}]
Since when all the parameters are zeros, the outputs of the function $f_{\Theta_{0}}(x)$ are also zeros. Then for $f \in \mathcal{F}_{r_{2}}$, it follows that
\begin{equation*}
\begin{aligned}
| \ell(f(x),y) - \ell(f_{\Theta_{0}}(x),y) | & \leq L_{\ell_{2}} \| f(x)-f_{\Theta_{0}}(x) \|_{2} = L_{\ell_{2}} \| f(x) \|_{2}, \\
\| \ell(f(x),y) - \ell(f_{\Theta_{0}}(x),y) \|_{L_{1}(P_{n})} & \leq L_{\ell_{2}} P_{n}\| f\|_{2} \leq L_{\ell_{2}}  \sqrt{r_{2}} \sqrt{\frac1n \sum_{i=1}^{n} \| x_{i} \|_{2}^{2} }.
\end{aligned}
\end{equation*}
Similarly, for $f \in \mathcal{F}_{r_{1}}$, it follows that
\begin{equation*}
\begin{aligned}
&| \ell(f(x),y) - \ell(f_{\Theta_{0}}(x),y) |  \leq L_{\ell_{\infty}} \| f(x)-f_{\Theta_{0}}(x) \|_{1} = L_{\ell_{\infty}} \| f(x) \|_{1}, \\
&\| \ell(f(x),y) - \ell(f_{\Theta_{0}}(x),y) \|_{L_{1}(P_{n})}  \leq L_{\ell_{\infty}} P_{n}\| f\|_{1} \leq L_{\ell_{\infty}} r_{1}R_{x}.
\end{aligned}
\end{equation*}
The proof is complete.
\end{proof}

\begin{proof}[Proof of Theorem \ref{rademacher}]
For the loss function class $\ell_{\mathcal{F}_{r_{2}}}$, the empirical Rademacher complexity can be bounded via Dudley's integral as
\begin{equation}\label{dudley}
\mathcal{R}_{S}(\ell_{\mathcal{F}_{r_{2}}}) \leq 12\int_{0}^{\delta_{\text{sup}}} \sqrt{ \frac{\log \mathcal{N}(\delta, \ell_{\mathcal{F}_{r_{2}}}, L_{1}(P_{n}) )}{n} }d\delta.
\end{equation}
We here use the $L_{1}(P_{n})$ metric instead of the $L_{2}(P_{n})$ metric in Dudley's integral. The proofs of Dudley's integral are similar except that one should use H$\ddot{o}$lder's inequality with $p=\infty, q=1$ under the $L_{1}(P_{n})$ metric and use Cauchy-Schwartz inequality under the $L_{2}(P_{n})$ metric. By Lemma \ref{lemma-variance-loss} and Lemma \ref{lemma-lipschitz}, one sees that the integral can be terminated at
$\delta_{\text{sup}} =  \sqrt{2r_{2}} \sqrt{1/n \sum_{i=1}^{n} \| x_{i} \|_{2}^{2} }$.
The covering number of the loss function space $\mathcal{N}(\delta, \ell_{\mathcal{F}_{r_{2}}}, L_{1}(P_{n}) )$ can be bounded by the covering number of the parameter space $\mathcal{N}(\delta_{\theta}, \varTheta, \norm{\cdot}_{F} )$ via the Lipschitz property with respect to parameters.
Specifically, let $\mathcal{C}_{\varTheta}$ be a $\delta_{\theta}$-cover of $\varTheta$, such that for each $\Theta \in \varTheta$, there exists $\Theta' \in \mathcal{C}_{\varTheta}$ satisfying $\norm{\Theta' - \Theta}_{F} \leq \delta_{\theta}$. According to Lemma \ref{lipschitz-parameter}, one has that
\begin{equation*}
\| \ell(f_{\Theta}(x),y) - \ell(f_{\Theta'}(x),y) \|_{L_{1}(P_{n})} \leq L^{\ell}_{1P_{n}} \norm{\Theta-\Theta'}_{F}  \leq  \delta_{\theta} L^{\ell}_{1P_{n}}.
\end{equation*}
Thus the function set $\mathcal{C}_{\ell_{\mathcal{F}_{r_{2}}}}=\{\ell(f_{\Theta'}(x),y):\Theta' \in \mathcal{C}_{\varTheta}  \}$ is a $\delta$-cover of $\ell_{\mathcal{F}_{r_{2}}}$ with $\delta=\delta_{\theta}L^{\ell}_{1P_{n}}$ and it holds that
\begin{equation*}
\mathcal{N}(\delta, \ell_{\mathcal{F}_{r_{2}}}, L_{1}(P_{n}) ) \leq  \mathcal{N}(\delta / L^{\ell}_{1P_{n}}, \varTheta, \norm{\cdot}_{F} ).
\end{equation*}
On the other hand, for the parameter set $\Theta \in \varTheta$, let $\Theta_{\text{vec}}$ denote the vector formed by the vectorized form of each parameter matrix arranging one by one. Denote the total number of parameters as $P$. By assumption, $\| \Theta_{\text{vec}} \|_{2} = \norm{\Theta}_{F} \leq R_{\Theta}$, the transformed parameter set is
$ \mathbb{B}_{2,R_{\Theta}}^{P} = \{\theta \in \mathbb{R}^{P}:  \| \theta \|_{2} \leq R_{\Theta}\} $ and 
\begin{equation*}
\mathcal{N}(\delta_{\theta}, \varTheta, \norm{\cdot}_{F} ) = \mathcal{N}(\delta_{\theta}, \mathbb{B}_{2,R_{\Theta}}^{P}, \| \cdot \|_{2} ) \leq \left( \frac{3R_{\Theta}}{\delta_{\theta}}\right)^{P}.
\end{equation*}
Then the Dudley's integral \eqref{dudley} is bounded as follows
\begin{equation*}%\label{eq-rsfr}
\begin{aligned}
\mathcal{R}_{S}(\ell_{\mathcal{F}_{r_{2}}}) %&\leq 12\int_{0}^{1} \sqrt{ \frac{\log \mathcal{N}(\delta, \ell_{\mathcal{F}}, L_{2}(P_{n}) }{n} }d\delta \\
%&\leq \frac{12}{\sqrt{n}}\int_{0}^{1} \sqrt{ \log \mathcal{N}(\delta / L_{2P_{n}}, \varTheta, \norm{\cdot}_{F} ) }d\delta \\
&\leq 12\sqrt{\frac{P}{n}} \int_{0}^{\delta_{\text{sup}}} \sqrt{\log \frac{  3R_{\Theta}L^{\ell}_{1P_{n}} }{\delta} }d\delta \\
&= 12\sqrt{\frac{P}{n}} \int_{0}^{\delta_{\text{sup}}} \sqrt{\log \frac{3R_{\Theta}L^{\ell}_{1P_{n}}}{\delta_{\text{sup}}} + \log \frac{\delta_{\text{sup}}}{\delta} }d\delta \\
&\leq 12\sqrt{\frac{P}{n}} \int_{0}^{\delta_{\text{sup}}} \left( \sqrt{\log \left| \frac{3R_{\Theta}L^{\ell}_{1P_{n}}}{\delta_{\text{sup}}} \right| } + \sqrt{\log \frac{\delta_{\text{sup}}}{\delta} } \right)d\delta \\
&=12\delta_{\text{sup}} \sqrt{\frac{P}{n}} \left( \sqrt{\log \left| \frac{3R_{\Theta}L^{\ell}_{1P_{n}}}{\delta_{\text{sup}}} \right| } +\sqrt{ \frac{\pi}{2} } \right) \\
&=12 \sqrt{\frac{2r_{2}}{n} \sum_{i=1}^{n} \| x_{i} \|_{2}^{2} } \sqrt{\frac{P}{n}} \left( \sqrt{\log \left| \frac{3R_{\Theta}L^{\ell}_{1P_{n}}}{\sqrt{2r_{2}/n \sum_{i=1}^{n} \| x_{i} \|_{2}^{2} }} \right| } +\sqrt{ \frac{\pi}{2} } \right).
\end{aligned}
\end{equation*}

For the loss function class $\ell_{\mathcal{F}_{r_{1}}}$, the empirical Rademacher complexity can be bounded similarly, except that the integral can be terminated at $\delta_{\text{sup}} =  r_{1}R_{x}$. The final result is that
\begin{equation*}
\begin{aligned}
\mathcal{R}_{S}(\ell_{\mathcal{F}_{r_{1}}}) \leq 12r_{1}R_{x} \sqrt{\frac{P}{n}} \left( \sqrt{\log \left| \frac{3R_{\Theta}L^{\ell}_{1P_{n}}}{r_{1}R_{x}} \right| } +\sqrt{ \frac{\pi}{2} } \right).
\end{aligned}
\end{equation*}
\end{proof}

\begin{proof}[Proof of Theorem \ref{rademacher-jacobian}]
For the function class $\mathcal{J}_{2\mathcal{F}_{r_{2}}}$, the empirical Rademacher complexity of $\mathcal{J}_{2\mathcal{F}_{r_{2}}}$ can also be bounded via Dudley's integral as
\begin{equation}\label{dudley2}
\mathcal{R}_{S}(\mathcal{J}_{2\mathcal{F}_{r_{2}}}) \leq 12\int_{0}^{r_{2}} \sqrt{ \frac{\log \mathcal{N}(\delta, \mathcal{J}_{2\mathcal{F}_{r_{2}}}, L_{1}(P_{n}) }{n} }d\delta.
\end{equation}
The integral can be terminated at $r_{2}$, since for $\|\nabla_{x}f\|_{F}^{2} \in \mathcal{J}_{2\mathcal{F}_{r_{2}}}$, it holds that $P_{n}\|\nabla_{x}f\|_{F}^{2}=\frac1n\sum_{i=1}^{n} \|\nabla_{x}f(x_{i})\|_{F}^{2} \leq r$, and thus the $\delta$-covering number of $\mathcal{J}_{2\mathcal{F}_{r_{2}}}$ under the $L_{1}(P_{n})$ metric is 1 when $\delta \geq r_{2}$. 
The covering number $\mathcal{N}(\delta, \mathcal{J}_{2\mathcal{F}_{r_{2}}}, L_{1}(P_{n}) )$ can be bounded by the covering number of the parameter space $\mathcal{N}(\delta_{\theta}, \varTheta, \norm{\cdot}_{F} )$.
Specifically, let $\mathcal{C}_{\varTheta}$ be a $\delta_{\theta}$-cover of $\varTheta$, such that for each $\Theta \in \varTheta$, there exists $\Theta' \in \mathcal{C}_{\varTheta}$ satisfying $\norm{\Theta' - \Theta}_{F} \leq \delta_{\theta}$. According to Lemma \ref{lipschitz-jacobian}, one has that
\begin{equation*}
\| \|\nabla_{x}f_{\Theta}(x)\|^{2}_{F} - \|\nabla_{x}f_{\Theta'}(x)\|^{2}_{F} \|_{L_{1}(P_{n})} \leq L^{F}_{1P_{n}} \norm{\Theta-\Theta'}_{F}  \leq  \delta_{\theta} L^{F}_{1P_{n}}.
\end{equation*}
Thus the function set $\mathcal{C}_{\mathcal{J}_{2\mathcal{F}_{r_{2}}}}=\{ \|\nabla_{x}f_{\Theta'}(x)\|^{2}_{F}:\Theta' \in \mathcal{C}_{\varTheta}  \}$ is a $\delta$-cover of $\mathcal{J}_{2\mathcal{F}_{r_{2}}}$ with $\delta=\delta_{\theta}L^{F}_{1P_{n}}$ and it holds that
\begin{equation*}
\mathcal{N}(\delta, \mathcal{J}_{2\mathcal{F}_{r_{2}}}, L_{1}(P_{n}) ) \leq  \mathcal{N}(\delta / L^{F}_{1P_{n}}, \varTheta, \norm{\cdot}_{F} ) \leq \left( \frac{3R_{\Theta}L^{F}_{1P_{n}}}{\delta}\right)^{P}.
\end{equation*}
Then the Dudley's integral \eqref{dudley2} is bounded as follows
\begin{equation*}
\begin{aligned}
\mathcal{R}_{S}(\mathcal{J}_{2\mathcal{F}_{r_{2}}}) &\leq 12 \sqrt{ \frac{P}{n} } \int_{0}^{r_{2}} \sqrt{ \log  \frac{3R_{\Theta}L^{F}_{1P_{n}}}{\delta} }d\delta \\
&=12 \sqrt{ \frac{P}{n} } \int_{0}^{r_{2}} \sqrt{ \log  \frac{3R_{\Theta}L^{F}_{1P_{n}}}{r_{2}}+ \log \frac{r_{2}}{\delta} }d\delta \\
&\leq 12 \sqrt{ \frac{P}{n} } \int_{0}^{r_{2}} \left(\sqrt{ \left| \log  \frac{3R_{\Theta}L^{F}_{1P_{n}}}{r_{2}} \right| }+  \sqrt{\log \frac{r_{2}}{\delta} } \right)d\delta \\
&= 12r_{2} \sqrt{ \frac{P}{n} } \left(\sqrt{ \left| \log  \frac{3R_{\Theta}L^{F}_{1P_{n}}}{r_{2}} \right| }+  \sqrt{\frac{\pi}{2} } \right).
\end{aligned}
\end{equation*}

For the function class $\mathcal{J}_{1\mathcal{F}_{r_{1}}}$, the empirical Rademacher complexity can also be bounded similarly, except that the integral can be terminated at $r_{1}$ and the Lipschitz constant is $L^{1}_{1P_{n}}$. The final result is that
\begin{equation*}
\begin{aligned}
\mathcal{R}_{S}(\mathcal{J}_{1\mathcal{F}_{r_{1}}}) \leq 12r_{1} \sqrt{ \frac{P}{n} } \left(\sqrt{ \left| \log  \frac{3R_{\Theta}L^{1}_{1P_{n}}}{r_{1}} \right| }+  \sqrt{\frac{\pi}{2} } \right).
\end{aligned}
\end{equation*}
\end{proof}

\bibliographystyle{elsarticle-harv} 
\bibliography{arxiv-241130.bbl}



%\begin{thebibliography}{00}
%
%%% For authoryear reference style
%%% \bibitem[Author(year)]{label}
%%% Text of bibliographic item
%
%\bibitem[Lamport(1994)]{lamport94}
%  \textit{\LaTeX: a document preparation system},
%  Addison Wesley, Massachusetts,
%\end{thebibliography}
\end{document}